\documentclass[journal]{IEEEtran}
\usepackage{amsmath,amsfonts}
\usepackage{amsthm}
\usepackage{algorithmic}
\usepackage{algorithm}
\usepackage{array}
\usepackage[caption=false,font=normalsize,labelfont=sf,textfont=sf]{subfig}
\usepackage{textcomp}
\usepackage{stfloats}
\usepackage{url}
\usepackage{verbatim}
\usepackage{graphicx}
\usepackage{cite}
\usepackage{booktabs}
\usepackage{multirow}
\usepackage{graphicx}

\usepackage[final,authormarkup=none,commandnameprefix=ifneeded]{changes}
\definechangesauthor[name={Zhiyaun}, color=red]{ZZY}
\setaddedmarkup{\textcolor{red}{#1}}

\usepackage{makecell}

\usepackage{bm}
\usepackage{hyperref}
\usepackage[table]{xcolor}
\usepackage{tabularx}

\theoremstyle{remark}
\newtheorem{remark}{Remark}

\makeatletter
\def\th@remark{%
  \normalfont 
  \thm@headfont{\bfseries\itshape}
}
\makeatother

\theoremstyle{definition}
\newtheorem{assumption}{Assumption}

\theoremstyle{definition}

\theoremstyle{plain} 
\newtheorem{proposition}{Proposition}

\begin{document}

\title{Canonical Policy: Learning Canonical 3D Representation for \added{$\mathrm{SE}(3)$-}Equivariant Policy}

\author{Zhiyuan Zhang$^{*}$, Zhengtong Xu$^{*}$, Jai Nanda Lakamsani, Yu She$^{\dagger}$
\thanks{$^{*}$ Equal Contribution, $^{\dagger}$ Corresponding Author.}
\thanks{Zhiyuan, Zhengtong, Jai, and Yu are with Purdue University, West Lafayette, USA  
{\tt\footnotesize zhan5570, xu1703, jlakamsa, shey@purdue.edu}}%
}

\maketitle

\begin{abstract}
Visual Imitation learning has achieved remarkable progress in robotic manipulation, yet generalization to unseen objects, scene layouts, and camera viewpoints remains a key challenge. Recent advances address this by using 3D point clouds, which provide geometry-aware, appearance-invariant representations, and by incorporating equivariance into policy architectures to exploit spatial symmetries. However, existing equivariant approaches often lack interpretability and rigor due to unstructured integration of equivariant components. We introduce canonical policy, a principled framework for 3D equivariant imitation learning that unifies 3D point cloud observations under a canonical representation. We first establish a theory of 3D canonical representations, enabling equivariant observation-to-action mappings by grouping both \added{seen} and \added{novel} point clouds to a canonical representation. We then propose a flexible policy learning pipeline that leverages geometric symmetries from canonical representation and the expressiveness of modern generative models. We validate canonical policy on 12 diverse simulated tasks and 4 real-world manipulation tasks across 16 configurations, involving variations in object color, shape, camera viewpoint, and robot platform.
Compared to state-of-the-art imitation learning policies, canonical policy achieves an average improvement of 18.0\% in simulation and \added{39.7}\% in real-world experiments, demonstrating superior generalization capability and sample efficiency.
For more details, please refer to the project website: \url{https://zhangzhiyuanzhang.github.io/cp-website/}.

\end{abstract}


\section{Introduction}
Imitation learning has made remarkable progress in robotic manipulation in recent years \cite{diffusion_policy,zhao2023learning,DP3,o2024open,wang2024gendp}. However, generalization and sample efficiency remain key challenges. In particular, visual imitation learning policies \cite{diffusion_policy,zhao2023learning} often struggle to generalize beyond the fixed dataset of human demonstrations on which they are trained. Unseen variations in object types, scene configurations, geometric layouts, or camera viewpoints can significantly degrade policy performance.

To improve generalization in visual imitation, recent research has explored the use of 3D point clouds as input observations \cite{DP3,iDP3,wang2024gendp,huang20243d}. Point clouds encode the geometric structure of the environment and are invariant to visual distractors such as background textures or object appearances. As a result, policies trained on 3D representations can more effectively learn the mapping between geometry of observations and robot actions, thereby improving both generalization and sample efficiency.

Complementary to this, another line of work enhances generalization and sample efficiency to spatial configurations by embedding equivariance into the network architecture \cite{Equibot,equidiff,EquiAct,gao2024riemann}. Since robotic manipulation takes place in 3D Euclidean space and involves tasks that are often equivariant to rotations and translations, incorporating symmetry priors into the policy architecture can significantly improve both sample efficiency and generalization performance.

However, existing approaches that aim to achieve equivariant policy learning from 3D point clouds or 2D images often rely on re-designing the entire network using equivariant neural modules \cite{Equibot,equidiff,EquiAct}. While this strategy introduces a certain degree of equivariance into the policy, it often lacks rigor, interpretability, and effectiveness in achieving equivariant learning. This is primarily because neural networks operate as black-box function approximators, and policy learning pipelines are typically long and complex, spanning from perception encoders to intermediate conditioning modules, and finally to generative action heads. \added{This makes it difficult to} ensure that equivariance is preserved in a principled and verifiable manner throughout the entire pipeline. As a result, the geometric symmetry of the learned policy and its sample efficiency are often degraded.

To enhance policy generalization and sample efficiency, we propose the canonical policy, which integrates point cloud representations with equivariance in a rigorous form. Our contributions can be summarized as follows:

1. We propose a theory of 3D canonical representations for imitation learning. Through rigorous theoretical derivations, we demonstrate how both \added{seen} and \added{novel} point clouds can be efficiently grouped under a unified canonical \added{representation}, and how this representation enables the construction of equivariant mappings from observations to actions. With this canonical representation, we show that 3D equivariant manipulation policies can be learned end-to-end from demonstrations, achieving strong spatial generalization in manipulation tasks.

2. Through a series of careful designs and theoretical analyses, we propose an imitation learning pipeline tailored for 3D canonical representations. This pipeline can flexibly integrate various forms of modern generative modeling, such as diffusion models and flow matching, enabling it to fully exploit the geometric symmetries introduced by canonical representations while leveraging the strong representational capacity of generative models for action distributions. 

3. We validate the effectiveness of canonical policy through extensive experiments. On a simulated benchmark of 12 diverse manipulation tasks, we compare against state-of-the-art 3D policy learning policies, including 3D diffusion policy (DP3)\cite{DP3}, improved 3D diffusion policy (iDP3)\cite{iDP3}, and EquiBot \cite{Equibot}. Canonical policy consistently outperforms all baselines, achieving an average task success improvement of 18\%.
We further evaluate canonical policy on four real-world manipulation tasks across 16 configurations involving variations in object color, shape, camera viewpoint, and robot platform. Canonical policy consistently outperforms state-of-the-art methods, including diffusion policy (DP)~\cite{diffusion_policy}, equivariant diffusion policy (EquiDiff)~\cite{equidiff}, and DP3~\cite{DP3}, achieving an average improvement of \added{39.7}\%.

\section{Related Works}

\subsection{Supervised Policy Learning}
Recent advances in imitation learning have enabled robots to perform complex, long-horizon manipulation tasks using supervised training frameworks \cite{diffusion_policy,chi2024universal,zhao2023learning,fu2024mobile}. This progress spans diverse areas, including scalable data collection pipelines \cite{zhao2023learning,fu2024mobile,wang2024dexcap}, the exploration of expressive policy architectures such as diffusion \cite{diffusion_policy}, flow matching \cite{flow_matching}, and VQVAE \cite{lee2024behavior}, as well as large-scale models driven by vision-language foundations or expanded datasets \cite{black2410pi0,team2023octo,zhao2024aloha}.

Most existing work has leveraged visual modalities such as RGB images \cite{diffusion_policy,chi2024universal,zhao2023learning} or integrated tactile signals  \cite{wang2024poco,yu2023mimictouch,xu2025unit,xue2025reactive,guzey2023see,lin2024learning,bhirangi2024anyskin} to enhance physical interaction modeling.
To improve robustness and generalization from limited demonstrations, recent work has explored structured data augmentation for visual policies, including semantic augmentations \cite{bharadhwaj2023roboagent} and embodiment- or viewpoint-level transformations \cite{chen2024roviaug}.
Meanwhile, 3D geometric representations such as point clouds \cite{DP3,iDP3,xue2025demogen,huang20243d,Equibot,yang2025fp33dfoundationpolicy,wang2024gendp} have shown particular promise for generalization. By capturing object shape while abstracting away distractors like color, background, and texture, 3D modalities allow policies to focus on task-relevant geometry. This enables more data-efficient learning and better transfer to unseen objects and environments, making 3D a particularly effective modality in scenarios that demand generalizable manipulation behavior. 

Canonical policy achieves policy-level equivariance by learning a 3D equivariant representation. Leveraging the inherent generalization capability of point clouds, particularly their ability to focus on geometry rather than appearance, canonical policy enables strong generalization to novel objects and scenes from only a few demonstrations. This spatial generalization further allows effective robotic manipulation across diverse tasks, objects, and environments, all under a single fixed-camera setup.

\subsection{Equivariance in Robot Learning}
Robots operate in a three-dimensional Euclidean space, where manipulation tasks naturally exhibit geometric symmetries such as rotations. A growing body of work~\cite{simeonov2023se,wang22onrobot,huang2023leveraging,pan2023tax,liu2023continual,jia2023seil,kim2023se,kohler2023symmetric,nguyen2024symmetry,eisner2024deep,hu2025push,zhao2025hierarchical,equidiff,actionflow,ipa} demonstrates that incorporating symmetry priors into policy learning can significantly enhance sample efficiency and overall performance. 

Moreover, point clouds are particularly well-suited as policy inputs for leveraging geometric symmetries, as they directly encode spatial structures and naturally transform under SE(3) operations. This alignment between input representation and task symmetry simplifies the design of equivariant policies and promotes better generalization across diverse object poses and environments. Recent works~\cite{EquiAct,Equibot,gao2024riemann,huang2024matchpolicysimplepipeline} have explored leveraging the natural alignment between point cloud representations and geometric symmetries to design equivariant policies for improved sample efficiency and generalization.

Our paper investigates the geometric symmetries of policies with point cloud input. By learning a canonical representation for 3D point cloud, canonical policy naturally integrates with widely used generative model-based policy heads, such as flow matching and diffusion models. This design inherits the powerful representation capabilities of generative models, the strong generalization ability of point clouds, and the sample efficiency benefits brought by geometric symmetries.

\section{Preliminaries}
This section provides the necessary background for our approach, which is situated in the context of visual-conditioned policy learning. We first specify the problem setup and formalize the input–output structure of the learning objective. We then introduce equivariance, which plays a central role in improving generalization and sample efficiency.
\added{Next, we establish the theoretical foundation of our canonical policy by introducing canonical representations, which map all samples within the same equivariant group to a unified representation. Finally, we contrast canonicalization with $\mathrm{SE}(3)$ data augmentation and clarify the generalization benefits of canonical representations.}

\subsection{Problem Setup}
We consider the problem of learning robotic control policies via behavior cloning. The goal is to train a policy that maps a sequence of observations to a corresponding sequence of actions, such that the resulting behavior closely imitates that of a demonstrated expert.

Let \( m \) denote the number of past observations and \( n \) the number of future actions. At each timestep \( t \), the input to the policy is an observation sequence
$
\mathcal{O} = \{\mathbf{o}_{t-m+1}, \ldots, \mathbf{o}_t\},
$
and the output is a predicted action sequence
$
\mathcal{A} = \{\mathbf{a}_t, \ldots, \mathbf{a}_{t+n-1}\}.
$
Each observation may include both visual inputs (e.g., images, voxels, or point clouds) and proprioceptive information (e.g., gripper pose or joint angles).

\subsection{Equivariance in Policy Learning}
Equivariance is a desirable property in policy learning, especially for robotic manipulation tasks that involve geometric transformations. In manipulation, the same task (e.g., grasping, placing) often needs to be performed under different object poses or camera viewpoints. If the policy can inherently adapt to such variations through equivariance, it reduces the need for extensive data augmentation or retraining. Prior work~\cite{simeonov2023se,huang2023leveraging,pan2023tax,Equibot,gao2024riemann} shows that incorporating symmetry priors significantly improves generalization and sample efficiency in robotic learning.
Formally, a function $f$ is equivariant to a transformation $\mathbf{h} \in \mathrm{SE}(3)$ if
\begin{equation*}
f(\mathbf{h} \cdot \mathbf{x}) = \mathbf{h} \cdot f(\mathbf{x}),
\end{equation*}
where $\mathbf{x}$ denotes the input.

In the context of policy learning, a policy $\pi$ maps a sequence of past observations $\mathcal{O}$ to a sequence of future actions $\mathcal{A}$.
We expect that if a transformation $\mathbf{h}$ is applied to the observation $\mathbf{o}\in \mathcal{O}$, the policy will produce a correspondingly transformed action sequence:
\begin{equation*}
\pi(\mathbf{h} \cdot \mathbf{o}) = \mathbf{h} \cdot \pi(\mathbf{o}).
\end{equation*}
This equivariance property allows the policy to generalize across variations in object poses, camera viewpoints, and scene configurations, as identical behaviors can be recovered from transformed observations. It improves sample efficiency and enhances robustness, as fewer demonstrations are required to cover task variations that exhibit underlying spatial symmetry.

\subsection{\added{Theory of Canonical Representations}}
\label{sec:theory}
\added{
This section establishes the theoretical foundation for canonical representations of 3D observations in visual-conditioned policy learning. Here, a policy is trained to predict actions from sequences of point cloud observations, typically collected from expert demonstrations of a specific manipulation task. Although these observations may vary in pose due to changes in viewpoint or object configuration, they often share the same underlying geometry. We aim to group such pose-varying observations into equivariant subsets that can be mapped to a common canonical form. This structure is formalized by the group-based assumption introduced below.}

\begin{figure}[t]
    \centering
    \includegraphics[width=0.9\linewidth]{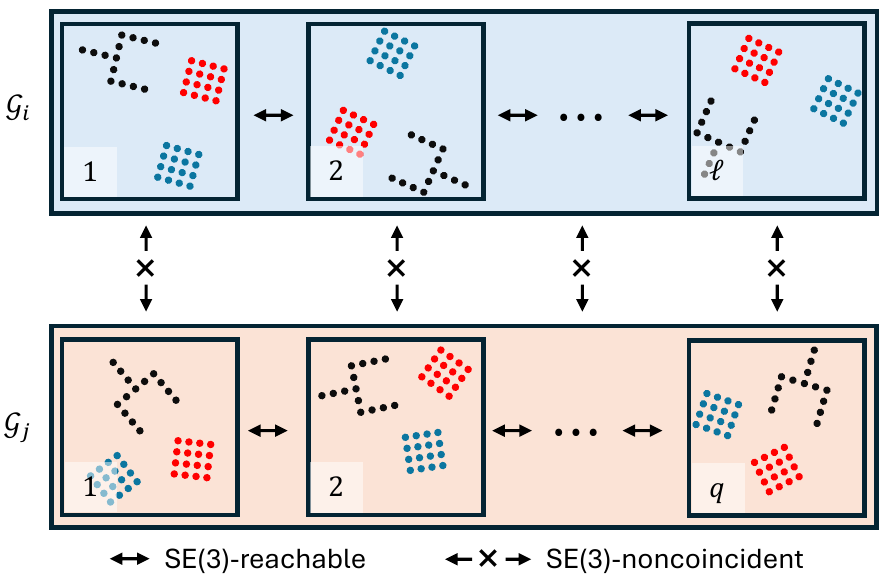}
    \caption{\added{Illustraion of two distinct equivariant groups $\mathcal{G}_i$ and $\mathcal{G}_j$. Samples within the same group can be aligned via an $\mathrm{SE}$(3) transformation, while samples from different groups cannot be made to coincide.}}
    \label{fig:g1g2}
\end{figure}

\begin{figure}[t]
    \centering
    \includegraphics[width=1.0\linewidth]{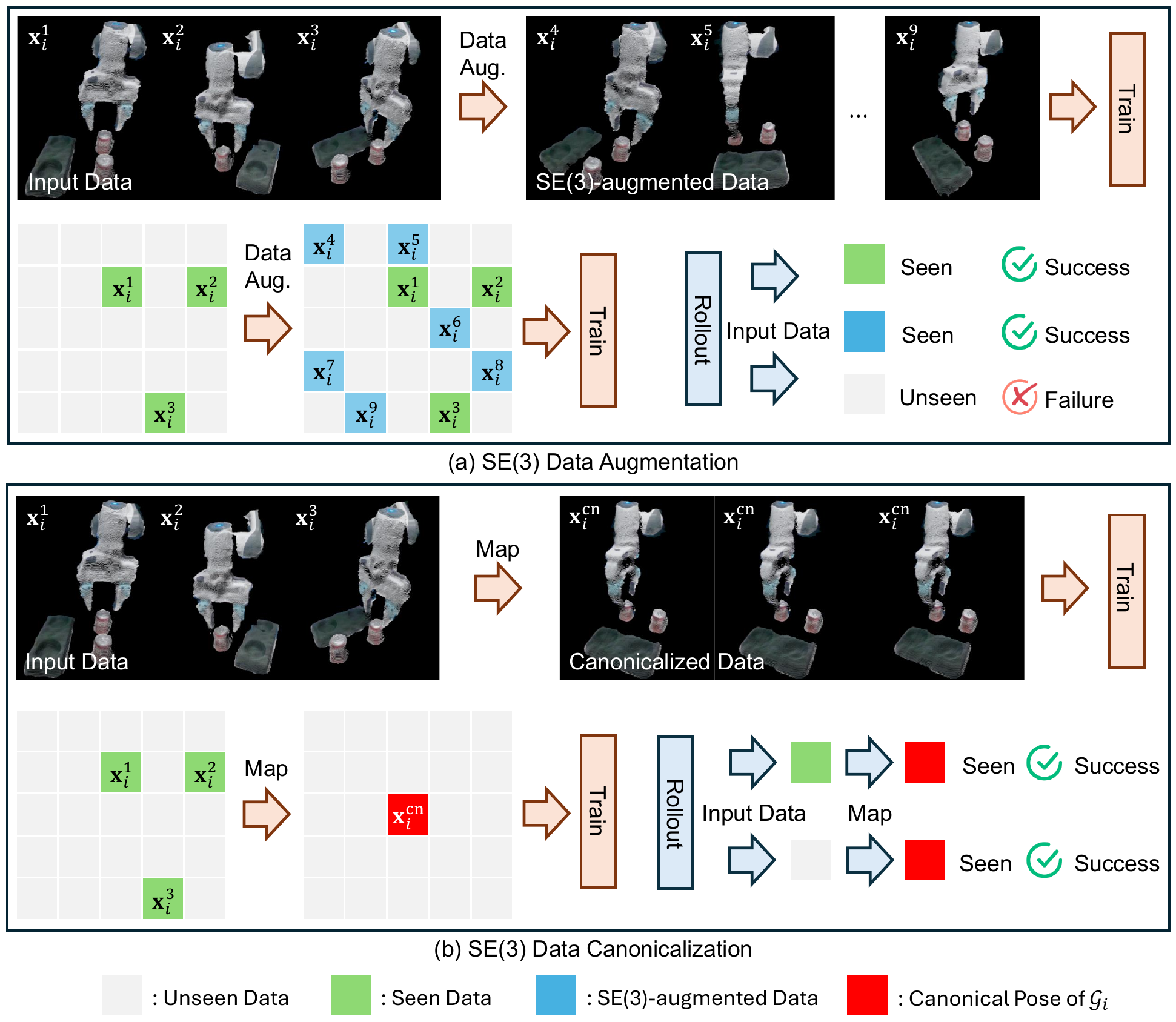}
    \caption{\added{Comparison of $\mathrm{SE}(3)$ augmentation and canonicalization: (a) $\mathrm{SE}(3)$ data augmentation; (b) $\mathrm{SE}(3)$ data canonicalization}}
    \label{fig:data_aug}
\end{figure}

\added{
\begin{assumption}
\label{asp:group}
A set of 3D point clouds can be partitioned into a finite number of mutually exclusive subsets:
$
\mathcal{D} = \{\mathcal{G}_1, \mathcal{G}_2, \cdots,\mathcal{G}_p\},
$
where $p$ is the total number of groups,
$\mathcal{G}_i=\{\mathbf{x}_i^{1},\mathbf{x}_i^{2},\cdots \mathbf{x}_i^{\ell}\}$ denotes the $i$-th group,
and $\mathbf{x}_i^{\ell}\in\mathbb{R}^{3\times N}$ is the $\ell$-th point cloud sample in the group $\mathcal{G}_i$, consisting of $N$ points.
For any two distinct groups \(\mathcal{G}_i\) and \(\mathcal{G}_j\), it holds that
\begin{equation*}
\forall \mathbf{x}_i^{\ell}\in\mathcal{G}_i,\; \mathbf{x}_j^q\in\mathcal{G}_j,\;\forall \mathbf{R}\in \mathrm{SO}(3), \ \mathbf{b} \in \mathbb{R}^{3 \times 1}, \mathbf{1}\in\mathbb{R}^{1\times N}
\end{equation*}
\begin{equation}
\label{group_diff}
\mathbf{x}_i^{\ell} \neq \mathbf{R} \mathbf{x}_j^q + \mathbf{b}\mathbf{1},
\end{equation}
where \( \mathbf{R} \), \( \mathbf{b} \), and $\mathbf{1}$ denote rotation matrix, translation vector, and all-ones vector, respectively.
For notational simplicity, we use $\mathbf{x},\,\mathbf{y}$ to denote generic point cloud samples and omit group and sample indices.
We also write $\mathbf{b}\mathbf{1}$ simply as $\mathbf{b}$ to represent the global translation.
Moreover, within each group, we have
\begin{equation}
\label{within_group}
\forall \mathbf{x}, \mathbf{y}\in \mathcal{G}_i,\ \exists \mathbf{R}, \mathbf{b}, \quad \mathbf{y} = \mathbf{R}\mathbf{x} + \mathbf{b}.
\end{equation}
\end{assumption}
}

\added{
The illustration of the groups $\mathcal{G}_i$ and $\mathcal{G}_j$ is shown in Fig.~\ref{fig:g1g2}. 
For any point cloud samples within the same group, there always exists an $\mathrm{SE}(3)$ transformation that can map one to the other. 
However, for point cloud samples belonging to two distinct groups $\mathcal{G}_i$ and $\mathcal{G}_j$, no $\mathrm{SE}(3)$ transformations can make them coincide.
}

\added{
\begin{remark}
\label{rmk:eq}
Under Assumption~\ref{asp:group}, we can see that \( \mathcal{G}_i\) forms an \emph{equivariant group}, such that for any $\mathrm{SE}$(3)-equivariant function \( f \),  the following holds:
\begin{equation*}
\label{equivariant_group}
\forall \mathbf{x}, \mathbf{y} \in \mathcal{G}_i,\ \exists \mathbf{R}, \mathbf{b},\quad f(\mathbf{y}) = f(\mathbf{R}\mathbf{x} + \mathbf{b}) = \mathbf{R}f(\mathbf{x}) + \mathbf{b}.
\end{equation*}
\end{remark}
}

\added{
In the context of visual-conditioned policy learning, it is reasonable to assume that new observations from the same task can be related to samples in $\mathcal{D}$ through $\mathrm{SE}$(3) transformations, which is summarized below.}

\added{
\begin{assumption}
\label{asp:ood}
For a set of point clouds $\mathcal{D}$ collected from a demonstrated task, any point cloud observation from the same task but outside $\mathcal{D}$  can be mapped to a point cloud within $\mathcal{D}$ through an $\mathrm{SE}$(3) transformation.
\end{assumption}
}

\added{
In practice, the collected dataset $\mathcal{D}$ is typically diverse enough to cover a wide range of variations.
Therefore, unseen point cloud observation but from the same task can be approximately mapped to $\mathcal{D}$ by $\mathrm{SE}$(3) transformations.
For example, in robotic manipulation tasks, a dataset may include multiple demonstrations of grasping the same object from different orientations and viewpoints. These demonstrations inherently capture variations such as rotations, translations, and perspective shifts, all of which can be effectively modeled as rigid body transformations in 3D space. Consequently, unseen geometric layouts can often be mapped to known layouts in the dataset through appropriate $\mathrm{SE}$(3) transformations.
}

\added{
To simulate observations with new poses not present in the original dataset, we define a transformed variant using $\mathrm{SE}$(3) transformations.
Specifically, for each group $\mathcal{G}_i$ and each sample $k$, let
$\mathbf{R}_i^k\!\in\!\mathrm{SO}(3)$ and $\mathbf{b}_i^k\!\in\!\mathbb{R}^{3\times 1}$
be randomly sampled. Define
\begin{equation*}
\mathcal{D}^{\mathrm{rd}} = \{\mathcal{G}_1^{\mathrm{rd}}, \ldots, \mathcal{G}_p^{\mathrm{rd}}\},\qquad
\mathcal{G}_i^{\mathrm{rd}} = \{\mathbf{R}_i^k \mathbf{x}_i^{k} + \mathbf{b}_i^k\}_{k=1}^{\ell}.
\end{equation*}
Here, the superscript ``rd” denotes randomly transformed representation.
}

\added{
Notably, Assumptions~\ref{asp:group} and \ref{asp:ood} still hold for $\mathcal{D}^{\mathrm{rd}}$.
Moreover, in most cases, data points generated by randomly sampled transformations are considered unseen data with respect to the original dataset, which we regard as seen data.
Based on this observation, we propose leveraging the equivariance property to map both seen and unseen samples to a canonical representation, which serves as a standardized representation for each group:
$
\mathcal{D}^{\mathrm{cn}} = \{\mathbf{x}_1^\mathrm{cn}, \mathbf{x}_2^\mathrm{cn}, \cdots, \mathbf{x}_p^\mathrm{cn}\},
$
where the superscript “cn” denotes canonical representation. Each \( \mathbf{x}_i^\mathrm{cn} \) denotes the canonical representation shared by both \( \mathcal{G}_i \) and its transformed counterpart \( \mathcal{G}_i^\mathrm{rd} \), such that
}
\added{
\begin{equation*}
\label{canonical}
\forall \mathbf{x} \in \mathcal{G}_i \cup \mathcal{G}_i^\mathrm{rd},\ \exists \mathbf{R}, \mathbf{b}, \quad \mathbf{x} = \mathbf{R} \mathbf{x}_i^\mathrm{cn} + \mathbf{b}.
\end{equation*}
Therefore, for any $\mathrm{SE}(3)$-equivariant function \( f \), it holds that:
}
\added{
\begin{equation}
\label{canonical_equivariant}
\forall\, \mathbf{x} \in \mathcal{G}_i \cup \mathcal{G}_i^{\mathrm{rd}},\  
\exists\, \mathbf{R}, \mathbf{b}, \quad
f(\mathbf{x}) = \mathbf{R} f(\mathbf{x}_i^{\mathrm{cn}}) + \mathbf{b}.
\end{equation}}
\added{
This implies that the equivariance property of \( f \) allows $\mathrm{SE}(3)$ transformations to be factored out from the original dataset, which can in turn benefit the training of a robot policy $\pi$. By leveraging this property and estimating the transformation parameters \( \mathbf{R} \) and \( \mathbf{b} \) as described in Equation~\eqref{canonical_equivariant}, the policy \( \pi \) can generalize any observation within the equivariant group as
}
\added{
\begin{equation}
\label{canonical_policy}
\forall \mathbf{x} \in \mathcal{G}_i \cup \mathcal{G}_i^{\mathrm{rd}},\ \exists \mathbf{R}, \mathbf{b}, \quad
\pi(\mathbf{x}) = \mathbf{R} \pi(\mathbf{x}_i^{\mathrm{cn}}) + \mathbf{b}.
\end{equation}
}
\added{
Consequently, once the transformation parameters are determined, both seen and unseen observations, whether from \( \mathcal{D} \) or \( \mathcal{D}^{\mathrm{rd}} \), can be reliably mapped to a canonical form.
The policy can thus operate consistently on the canonical representation, allowing the policy to achieve generalizable performance.}

\subsection{\added{Differences Between Canonical Representation and SE(3) Data Augmentation}}
\added{
To better understand the benefits of canonical representations, we provide a conceptual comparison with $\mathrm{SE}(3)$ data augmentation, which is a commonly used technique for improving generalization to rigid pose variations in the data.}

\added{
As illustrated in Fig.~\ref{fig:data_aug}(a), $\mathrm{SE}(3)$ data augmentation operates by applying random rigid transformations to the input data.
The augmented samples, together with the original data, are then used to extract observation features, with the aim of improving inference-time robustness by exposing the model to diverse poses.
In contrast, Fig.~\ref{fig:data_aug}(b) shows that canonical representation explicitly normalizes inputs to a consistent form, regardless of their original transformation, before passing them to the downstream network.}

\added{
However, data augmentation has fundamental limitations in pose generalization. Since the augmented samples are finite and sampled from a subset of all possible rigid transformations, the trained policy may fail to generalize to unseen poses that were not covered during training (see the 'Rollout' stage in Fig.~\ref{fig:data_aug}(a)). This leads to degraded performance when encountering new viewpoints or spatial configurations of the same object. In contrast, the canonical representation approach addresses this limitation by utilizing an equivariant mapping that transforms any instance in the same equivariant group into a unique canonical pose. As a result, the network always observes the same canonicalized input, enabling better generalization even to unseen but transformation-equivalent inputs (see the 'Rollout' stage in Fig.~\ref{fig:data_aug}(b)).}

\added{
Formally, canonical representation can be seen as a form of pre-alignment that enforces invariance to a certain group of transformations. Instead of requiring the network to learn this invariance implicitly through data diversity, the equivariant transformation module explicitly reduces the input variation.
This enables the policy to learn action mappings from a standardized input space.}

\added{To enhance readability, Table~\ref{tab:notation} summarizes the key notation used throughout the paper before we proceed to the Method section.}

\begin{table}[t]
\renewcommand{\arraystretch}{1.5}
\caption{\added{Key notation used in the paper.}}
\label{tab:notation}
\centering
\scriptsize
\begin{tabularx}{\linewidth}{@{}>{\raggedright\arraybackslash}p{0.28\linewidth} X@{}}
\toprule
\textbf{Symbol} & \textbf{Meaning} \\
\midrule
$\mathcal{D},\;\mathcal{G}$ & A set of 3D point clouds and an equivariant group (also a subset of $\mathcal{D}$). \\

$\mathbf{R}\in \mathrm{SO}(3),\;\mathbf{b}\in\mathbb{R}^3$ & Rotation matrix and translation vector. \\

$\mathbf{x}_i^{\ell} \in \mathbb{R}^{3\times N}$ &
The \(\ell\)-th point cloud sample in group \(\mathcal{G}_i\); a \(3\times N\) array whose \(N\) columns are 3D points. \\

$\mathbf{x},\;\mathbf{y}$ & Generic point cloud samples. \\

$\mathbf{o},\;\mathbf{a},\;\mathbf{s}$ & Observation, action, and robot proprioceptive state. \\

$f,\;\pi$ & Equivariant function and policy mapping. \\

``rd”, ``cn”, ``mn”, ``de” & Randomly transformed representation, canonical representation, mean operator, and decentered representation. \\

``pos”, ``ori”, ``grip” & End-effector position, orientation, and gripper open width. \\
\bottomrule
\end{tabularx}
\end{table}

\section{Method}
\label{sec:method}
In this section,
\added{we first present the data canonicalization procedure, which estimates the \(\mathrm{SE}(3)\) transformation that maps each sample to a canonical frame. We then introduce the canonical policy framework, which achieves \(\mathrm{SE}(3)\)-equivariant visual imitation by mapping observations, robot states, and actions into a shared canonical frame.}
Finally, we describe the point cloud encoder designed for feature extraction.

\subsection{SE(3) Equivariance via Data Canonicalization}
\label{sec:preprocessing}
\added{The formulation in Section~\ref{sec:theory} builds the theoretical foundation of canonical representations for 3D observations, enabling the policy to generalize to unseen observations that are SE(3)-equivariant to those seen during training. The key to canonical representation is determining the transformation parameters in Eq.~\eqref{canonical_policy}.}
\added{Therefore,} \added{in} this section, we formally derive how such a canonical pose can be computed via an $\mathrm{SE}(3)$-equivariant canonicalization procedure. Specifically, we describe how the rotation matrix is estimated using an $\mathrm{SO}(3)$-equivariant network, enabling point clouds related by rigid motions to be mapped to a consistent canonical frame.

According to Assumption~\ref{asp:group}, suppose \( \mathbf{x}, \mathbf{y} \in \mathcal{G}_i \) are related by a rigid transformation, i.e., \( \mathbf{y} = \mathbf{R} \mathbf{x} + \mathbf{b} \). The decentered form of \( \mathbf{y} \) is then given by:
\begin{align*}
\mathbf{y}^\mathrm{de}&=\mathbf{y}-\mathbf{y}^\mathrm{mn} \\
&=\mathbf{R}\mathbf{x}+\mathbf{b}-(\mathbf{R}\mathbf{x}^\mathrm{mn}+\mathbf{b}) \\
&=\mathbf{R}(\mathbf{x}-\mathbf{x}^\mathrm{mn}) \\
&=\mathbf{R}\mathbf{x}^\mathrm{de},
\end{align*}
\added{where the superscripts ``de'' and ``mn'' denote the decentered representation and the mean operator, respectively.
}
This transformation effectively eliminates the effect of translation, making it translation invariance, ensuring that the relationship between any two decentered elements is solely determined by a rotation.

Let \( \Phi \) be an arbitrary SO(3)-equivariant network that satisfies the equivariance property:
$
\Phi(\mathbf{R} \mathbf{x}) = \mathbf{R} \Phi(\mathbf{x}).
$
Applying this property to the decentered forms \( \mathbf{x}^\mathrm{de} \) and \( \mathbf{y}^\mathrm{de} \), which belong to the same equivariant group, yields:
$
\Phi(\mathbf{y}^\mathrm{de}) = \mathbf{R} \Phi(\mathbf{x}^\mathrm{de}).
$
Suppose the output feature \( \Phi(\mathbf{x}^\mathrm{de}) \) consists of two rotation-equivariant vectors, denoted as \( \mathbf{r}_{\mathbf{x}}^1 \in \mathbb{R}^{3 \times 1} \) and \( \mathbf{r}_{\mathbf{x}}^2 \in \mathbb{R}^{3 \times 1} \), i.e.,
\begin{equation}
\label{eq:Phi_x}
\Phi(\mathbf{x}^\mathrm{de}) = \{ \mathbf{r}_{\mathbf{x}}^1, \mathbf{r}_{\mathbf{x}}^2 \},
\end{equation}
then the output corresponding to \( \mathbf{y}^\mathrm{de} \) must take the form:
\begin{equation}
\label{eq:Phi_y}
\Phi(\mathbf{y}^\mathrm{de})
=\{\mathbf{r}_{\mathbf{y}}^1, \mathbf{r}_{\mathbf{y}}^2\}
=\{\mathbf{R} \mathbf{r}_{\mathbf{x}}^1, \mathbf{R} \mathbf{r}_{\mathbf{x}}^2\}.
\end{equation}

A natural instantiation of such an SO(3)-equivariant network is the Vector Neuron framework~\cite{VN}, as it operates directly on vector features rather than scalars.
Specifically, Vector Neuron represents each feature as a 3D vector and aggregates them into a \( 3 \times D \) tensor. Although each column corresponds to a distinct equivariant feature component, they are not transformed independently. Instead, the entire tensor undergoes a shared transformation under a single rotation matrix, preserving global rotational equivariance.
As a result, Equations~\eqref{eq:Phi_x} and~\eqref{eq:Phi_y} hold by design.

\begin{proposition}
Let \( \mathbf{x}^\mathrm{de}, \mathbf{y}^\mathrm{de} \) denote two decentered elements within the same equivariant group, and \( \mathbf{R}_{\mathbf{x}}, \mathbf{R}_{\mathbf{y}} \in \mathrm{SO}(3) \) be the rotation matrices constructed via Schmidt orthogonalization from the respective SO(3)-equivariant network outputs. Then, the canonicalized representations,
\[
{\mathbf{x}}^\mathrm{cn} = \mathbf{R}_{\mathbf{x}}^{-1} \mathbf{x}^\mathrm{de}, \quad
{\mathbf{y}}^\mathrm{cn} = \mathbf{R}_{\mathbf{y}}^{-1} \mathbf{y}^\mathrm{de},
\]
are equal, i.e., \( {\mathbf{x}}^\mathrm{cn} = {\mathbf{y}}^\mathrm{cn} \).
\end{proposition}

\begin{proof}
We first normalize the first vector:
\begin{equation}
\label{ux1}
\mathbf{u}^1_{\mathbf{x}} = \frac{\mathbf{r}_{\mathbf{x}}^1}{\|\mathbf{r}_{\mathbf{x}}^1\|}, 
\end{equation}
then orthogonalize the second vector with respect to the first and normalize it:
\[
\mathbf{u}^2_{\mathbf{x}} = \frac{\mathbf{r}_{\mathbf{x}}^2 - ({\mathbf{u}^1_{\mathbf{x}}}^\top \mathbf{r}_{\mathbf{x}}^2) \mathbf{u}^1_{\mathbf{x}}}{\left\| \mathbf{r}_{\mathbf{x}}^2 - ({\mathbf{u}^1_{\mathbf{x}}}^\top \mathbf{r}_{\mathbf{x}}^2) \mathbf{u}^1_{\mathbf{x}} \right\|},
\]
and finally construct the third orthonormal basis vector via the cross product:
$
\mathbf{u}^3_{\mathbf{x}} = \mathbf{u}^1_{\mathbf{x}} \times \mathbf{u}^2_{\mathbf{x}}.
$
The resulting rotation matrix corresponding to $\mathbf{x}^\mathrm{de}$ is then given by:
\begin{equation}
\label{eq:rotation_matrix}
\mathbf{R}_{\mathbf{x}} = [\mathbf{u}^1_{\mathbf{x}},\ \mathbf{u}^2_{\mathbf{x}},\ \mathbf{u}^3_{\mathbf{x}}].
\end{equation}
This procedure is inherently SO(3)-equivariant, since normalization, projection subtraction, and cross product are all equivariant under rotation.

Similarly, given the equivariant outputs \( \mathbf{r}_{\mathbf{y}}^1 = \mathbf{R} \mathbf{r}_{\mathbf{x}}^1 \) and \( \mathbf{r}_{\mathbf{y}}^2 = \mathbf{R} \mathbf{r}_{\mathbf{x}}^2 \), we construct the rotation matrix \( \mathbf{R}_{\mathbf{y}} \in \mathrm{SO}(3) \) following the same Schmidt orthogonalization process. Specifically, we have:
\begin{align}
\mathbf{u}^1_{\mathbf{y}} &= \frac{\mathbf{r}_{\mathbf{y}}^1}{\|\mathbf{r}_{\mathbf{y}}^1\|} = \frac{\mathbf{R} \mathbf{r}_{\mathbf{x}}^1}{\|\mathbf{R} \mathbf{r}_{\mathbf{x}}^1\|} = \mathbf{R} \frac{\mathbf{r}_{\mathbf{x}}^1}{\|\mathbf{r}_{\mathbf{x}}^1\|} = \mathbf{R} \mathbf{u}^1_{\mathbf{x}}, \label{eq:uy1} \\
\mathbf{u}^2_{\mathbf{y}} &= \frac{\mathbf{r}_{\mathbf{y}}^2 - ({\mathbf{u}^1_{\mathbf{y}}}^{\top}
\mathbf{r}_{\mathbf{y}}^2)\mathbf{u}^1_{\mathbf{y}}}{\left\|\mathbf{r}_{\mathbf{y}}^2 - ({\mathbf{u}^1_{\mathbf{y}}}^{\top} \mathbf{r}_{\mathbf{y}}^2)\mathbf{u}^1_{\mathbf{y}}\right\|} \notag \\
&= \frac{\mathbf{R} \mathbf{r}_{\mathbf{x}}^2 - (\mathbf{R} \mathbf{u}^1_{\mathbf{x}})^\top \mathbf{R} \mathbf{r}_{\mathbf{x}}^2 \mathbf{R} \mathbf{u}^1_{\mathbf{x}}}{\left\| \mathbf{R} \mathbf{r}_{\mathbf{x}}^2 - (\mathbf{R} \mathbf{u}^1_{\mathbf{x}})^\top \mathbf{R} \mathbf{r}_{\mathbf{x}}^2 \mathbf{R} \mathbf{u}^1_{\mathbf{x}} \right\|} \notag \\
&= \frac{\mathbf{R} \left( \mathbf{r}_{\mathbf{x}}^2 - ({\mathbf{u}^1_{\mathbf{x}}}^{\top} \mathbf{r}_{\mathbf{x}}^2) \mathbf{u}^1_{\mathbf{x}} \right)}{\left\| \mathbf{R} \left( \mathbf{r}_{\mathbf{x}}^2 - ({\mathbf{u}^1_{\mathbf{x}}}^{\top} \mathbf{r}_{\mathbf{x}}^2) \mathbf{u}^1_{\mathbf{x}} \right) \right\|} \notag \\
&= \mathbf{R} \frac{ \mathbf{r}_{\mathbf{x}}^2 - ({\mathbf{u}^1_{\mathbf{x}}}^{\top} \mathbf{r}_{\mathbf{x}}^2) \mathbf{u}^1_{\mathbf{x}} }{ \left\| \mathbf{r}_{\mathbf{x}}^2 - ({\mathbf{u}^1_{\mathbf{x}}}^{\top} \mathbf{r}_{\mathbf{x}}^2) \mathbf{u}^1_{\mathbf{x}} \right\| } = \mathbf{R} \mathbf{u}^2_{\mathbf{x}}, \notag \\
\mathbf{u}^3_{\mathbf{y}} &= \mathbf{u}^1_{\mathbf{y}} \times \mathbf{u}^2_{\mathbf{y}} = (\mathbf{R} \mathbf{u}^1_{\mathbf{x}}) \times (\mathbf{R} \mathbf{u}^2_{\mathbf{x}}) = \mathbf{R} (\mathbf{u}^1_{\mathbf{x}} \times \mathbf{u}^2_{\mathbf{x}}) = \mathbf{R} \mathbf{u}^3_{\mathbf{x}}. \notag
\end{align}
Then the rotation matrix corresponding to $\mathbf{y}^\mathrm{de}$ is given by:
\begin{equation*}
\mathbf{R}_{\mathbf{y}} = [\mathbf{u}^1_{\mathbf{y}},\ \mathbf{u}^2_{\mathbf{y}},\ \mathbf{u}^3_{\mathbf{y}}]
=\mathbf{R}[\mathbf{u}^1_{\mathbf{x}},\ \mathbf{u}^2_{\mathbf{x}},\ \mathbf{u}^3_{\mathbf{x}}]
=\mathbf{R}\mathbf{R}_{\mathbf{x}}.
\end{equation*}

Thus, by computing the canonical pose of \( \mathbf{x}^\mathrm{de} \) via inverse rotation with \( \mathbf{R}_{\mathbf{x}} \), we obtain:
$
{\mathbf{x}}^\mathrm{cn} = \mathbf{R}_{\mathbf{x}}^{-1} \mathbf{x}^\mathrm{de}.
$
Similarly, the canonical pose of \( \mathbf{y}^\mathrm{de} \) is given by:
\begin{equation*}
{\mathbf{y}}^\mathrm{cn} = \mathbf{R}_{\mathbf{y}}^{-1} \mathbf{y}^\mathrm{de} = (\mathbf{R} \mathbf{R}_{\mathbf{x}})^{-1} \mathbf{R} \mathbf{x}^\mathrm{de} = \mathbf{R}_{\mathbf{x}}^{-1} \mathbf{x}^\mathrm{de} = {\mathbf{x}}^\mathrm{cn}.
\end{equation*}
\end{proof}
This result implies that \( {\mathbf{y}}^\mathrm{cn} = {\mathbf{x}}^\mathrm{cn} \); in other words, all elements within the same equivariant group are mapped to a unique canonical representation.

\begin{figure}[t]
    \centering
    \includegraphics[width=1.0\linewidth]{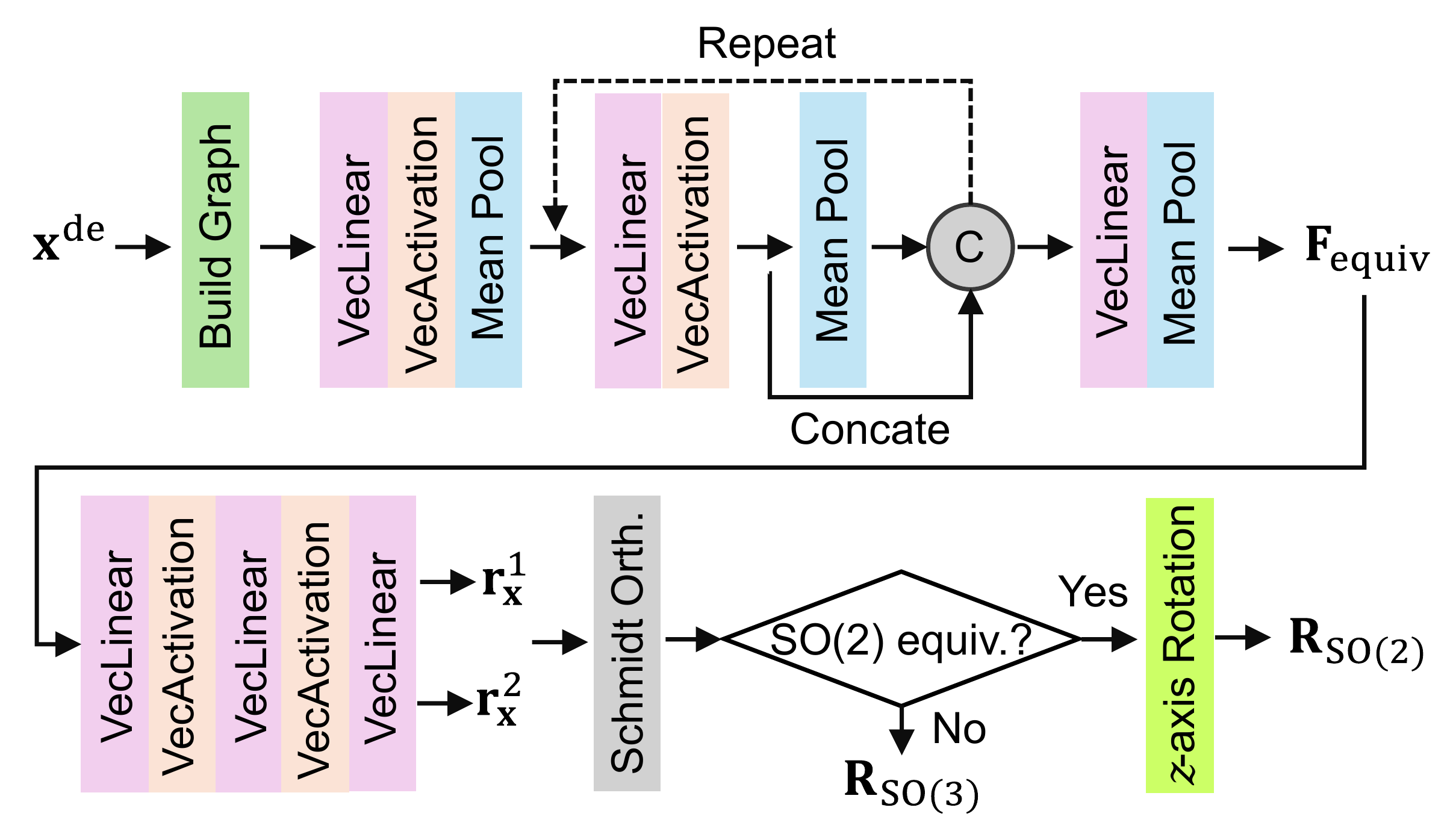}
    \caption{\added{Vector Neuron framework for estimating $\mathrm{SO}(3)$-equivariant rotation matrix. A graph is built from the decentered point cloud, local features are aggregated into global equivariant features, and Schmidt orthogonalization generates a rotation in $\mathrm{SO}(3)$ or $\mathrm{SO}(2)$ that aligns the input to a canonical frame.}}
    \label{fig:SO3}
\end{figure}

In this case, the data within the same group will eventually convert to the identical representation. If the policy is trained on such a canonical pose, then when it encounters a completely \added{unseen} data that is a rotated or translated version of a previously seen sample, the canonical mapping step will transform it into the same canonical form of the seen one. Therefore, the policy can seamlessly generalize across variations in rotation and translation, significantly improving its robustness.

While our canonicalization framework is derived to be fully $\mathrm{SO}(3)$-equivariant, we observe that, by convention, the reconstructed point clouds in both simulated and real-world settings consistently have their $z$-axis aligned vertically upward. As a result, out-of-distribution variations are effectively limited to in-plane $\mathrm{SO}(2)$ rotations about the $z$-axis. 

Although $\mathrm{SO}(2) \subset \mathrm{SO}(3)$, explicitly handling only $\mathrm{SO}(2)$ variations can be beneficial in practice. It simplifies the canonicalization process and eliminates residual variation caused by estimating full 3D orientations when only in-plane rotation is present. 

Let \( \mathbf{y}^\mathrm{de} = \rho(\theta) \mathbf{x}^\mathrm{de} \), where \( \rho(\theta) \in \mathrm{SO}(2) \subset \mathrm{SO}(3)\) denotes a rotation about the \( z \)-axis, by angle $\theta$. Let an $\mathrm{SO}(3)$-equivariant network estimate rotation matrices \( \mathbf{R}_{\mathbf{x}}, \mathbf{R}_{\mathbf{y}} \in \mathrm{SO}(3) \) from \( \mathbf{x}^\mathrm{de} \) and \( \mathbf{y}^\mathrm{de} \), respectively.
\begin{proposition}
Let \( \rho(\theta_{\mathbf{x}}) \) and \( \rho(\theta_{\mathbf{y}}) \) denote the rotation matrices obtained by extracting the in-plane rotational components, i.e., the rotations around the \( z \)-axis, from the corresponding full rotation matrices \( \mathbf{R}_{\mathbf{x}} \) and \( \mathbf{R}_{\mathbf{y}} \), respectively.
Then, the canonical poses
\[
{\mathbf{x}}^\mathrm{cn} = \rho^{-1}(\theta_\mathbf{x}) \mathbf{x}^\mathrm{de}, \quad
{\mathbf{y}}^\mathrm{cn} = \rho^{-1}(\theta_\mathbf{y}) \mathbf{y}^\mathrm{de}
\]
are equal, i.e., \( {\mathbf{x}}^\mathrm{cn} = {\mathbf{y}}^\mathrm{cn} \).
\end{proposition}

\begin{proof}
Let the first column of \( \mathbf{R}_{\mathbf{x}} \) be \( \mathbf{u}_{\mathbf{x}}^1 = [x_1, y_1, z_1]^\top \). Since \( \mathbf{y}^\mathrm{de} = \rho(\theta) \mathbf{x}^\mathrm{de} \), and according to Equation~\eqref{eq:uy1}, the first column of \( \mathbf{R}_{\mathbf{y}} \) satisfies:
\begin{align*}
\mathbf{u}_{\mathbf{y}}^1 
&= \rho(\theta) \mathbf{u}_{\mathbf{x}}^1 \\
&=
\begin{bmatrix}
\cos\theta & -\sin\theta & 0 \\
\sin\theta & \cos\theta & 0 \\
0 & 0 & 1
\end{bmatrix}
\begin{bmatrix}
x_1 \\
y_1 \\
z_1
\end{bmatrix} \\
&=
\begin{bmatrix}
x_1 \cos\theta - y_1 \sin\theta \\
x_1 \sin\theta + y_1 \cos\theta \\
z_1
\end{bmatrix}.
\end{align*}
We extract the in-plane rotation angles by projecting $\mathbf{R}_\mathbf{x}$ and $\mathbf{R}_\mathbf{y}$ onto the \( xy \)-plane and applying the \( \mathrm{atan2} \) function:
\begin{align*}
\theta_{\mathbf{x}} &= \mathrm{atan2}(y_1, x_1), \\
\theta_{\mathbf{y}} &= \mathrm{atan2}(x_1 \sin\theta + y_1 \cos\theta,\ x_1 \cos\theta - y_1 \sin\theta).
\end{align*}
Using the rotational property of \( \mathrm{atan2} \), we observe: $\theta_{\mathbf{y}} = \theta_{\mathbf{x}} + \theta.$
Then, the canonicalized form of \( \mathbf{y}^\mathrm{de} \) becomes:
\[
{\mathbf{y}}^\mathrm{cn} = \rho^{-1}(\theta_{\mathbf{y}})\mathbf{y}^\mathrm{de} 
= \rho(-\theta_{\mathbf{x}} - \theta) \rho(\theta) \mathbf{x}^\mathrm{de} 
= \rho(-\theta_{\mathbf{x}}) \mathbf{x}^\mathrm{de} 
= {\mathbf{x}}^\mathrm{cn}.
\]
Hence, \( {\mathbf{x}}^\mathrm{cn} = {\mathbf{y}}^\mathrm{cn} \), completing the proof.
\end{proof}

Similar to the $\mathrm{SO}(3)$ case, this result demonstrates that all observations related by $\mathrm{SO}(2)$ transformations are mapped to a consistent canonical pose. This guarantees rotational consistency within the equivariant group and enables the policy to generalize reliably to unseen in-plane rotations at deployment.

\added{
Figure~\ref{fig:SO3} illustrates the detailed pipeline for estimating the 
\(\mathrm{SO}(3)\)-equivariant rotation matrix using the Vector Neuron architecture. 
The network takes the decentered point cloud \(\mathbf{x}^{\mathrm{de}}\) as input and then constructs a k-nearest neighbor (KNN) graph to capture local geometric structure 
and derive features such as relative offsets, absolute positions, and local directions, thereby enhancing robustness to noise. 
Stacked \(\mathrm{SO}(3)\)-equivariant linear layers and activation blocks with mean pooling 
yield a global equivariant feature \(\mathbf{F}_{\text{equiv}}\), 
which is then used to compute two \(\mathrm{SO}(3)\)-equivariant vectors 
\(\{\mathbf{r}^1_{\mathbf{x}},\,\mathbf{r}^2_{\mathbf{x}}\}\). 
Applying Schmidt orthogonalization to these vectors produces an 
\(\mathrm{SO}(3)\) rotation matrix. 
For tasks with only an \(\mathrm{SO}(2)\) ambiguity, the network further extracts 
the in-plane (z-axis) rotation. 
This rotation maps \(\mathbf{x}^{\mathrm{de}}\) to its canonicalized counterpart.
}

\added{
The reliability of the $\mathrm{SO(3)}$-equivariant network under noisy conditions, parameter sensitivity, and feature visualization will be systematically examined in Section~\ref{sec:sim}.
}

\subsection{\added{Canonical Policy Pipeline}}
\label{sec:overview}
\begin{figure*}[t]
    \centering
    \includegraphics[width=1.0\linewidth]{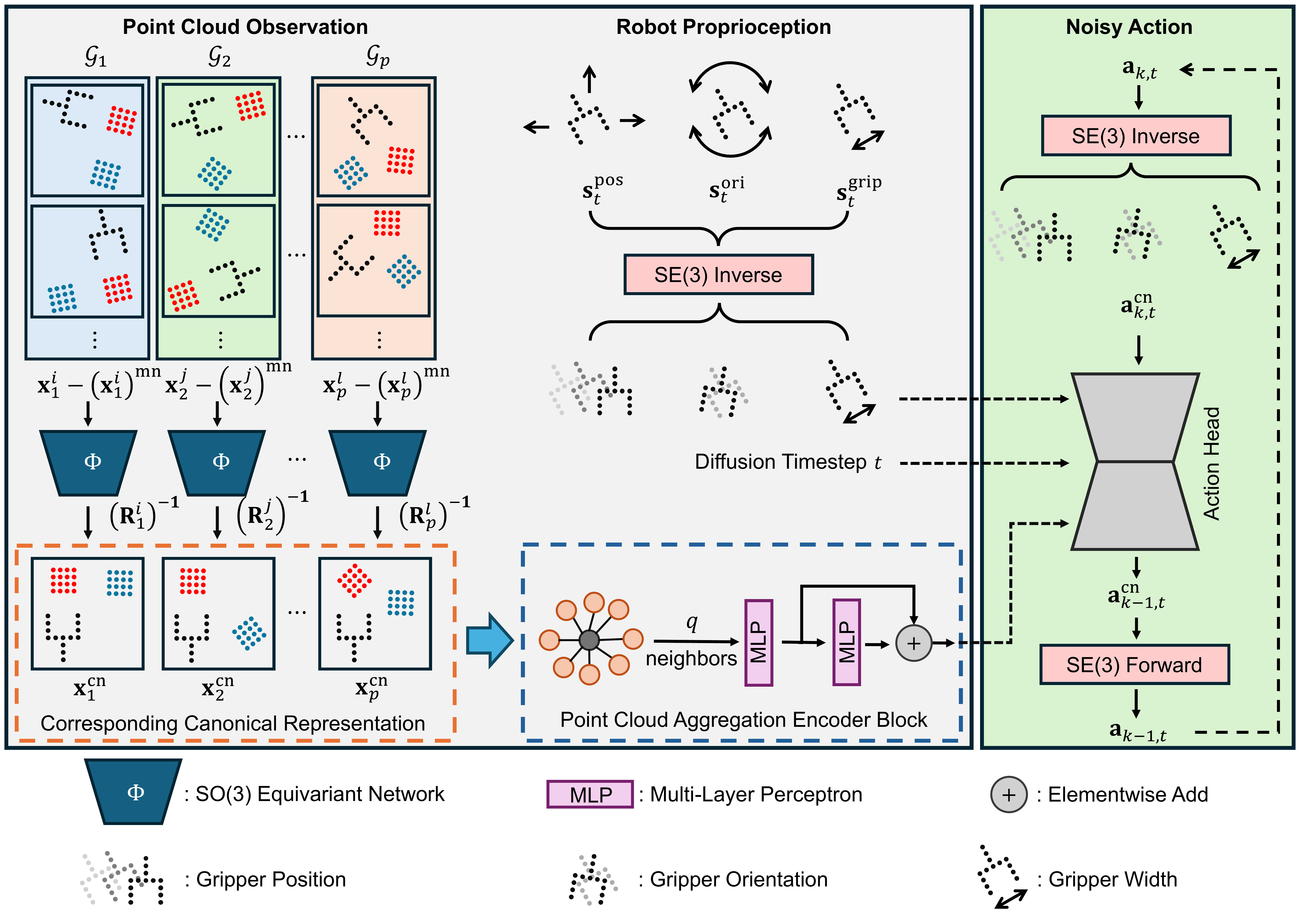}
    \caption{Overview of the canonical policy. The input point cloud \( \mathcal{G} \) is first centered by subtracting its mean \( \mathcal{G}^\mathrm{mn} \), and then processed by an \( \mathrm{SO}(3) \)-equivariant network \( \Phi \) to estimate the object rotation \( \mathbf{R} \). This rotation is subsequently used to obtain the canonicalized point cloud \( \mathcal{G}^\mathrm{cn} \), where the superscript ``cn'' denotes the canonical representation. These canonicalized point clouds are then encoded using a point cloud aggregation encoder. At each diffusion timestep \( t \), robot proprioception, including end-effector position \( \mathbf{s}^{\mathrm{pos}}_t \), orientation \( \mathbf{s}^{\mathrm{ori}}_t \), and gripper width \( \mathbf{s}^{\mathrm{grip}}_t \), is transformed into the canonical frame via an SE(3) inverse transformation, yielding canonical proprioception. A similar canonicalization is applied to the noisy action \( \mathbf{a}_{k,t} \), after which the model predicts canonical actions \( {\mathbf{a}}_{k-1,t}^\mathrm{cn} \). The output is finally mapped back to the original observation frame through an SE(3) forward transformation to obtain \( \mathbf{a}_{k-1,t} \).}
    \label{fig:Pipeline}
\end{figure*}

\added{Following the data canonicalization described in Section~\ref{sec:preprocessing}, this section presents the architecture of the canonical policy, which performs transformation-invariant action prediction by operating in a canonical coordinate frame. The core idea is to map both visual and low-dimensional inputs (e.g., robot proprioceptive state and action) to a consistent pose, reducing geometric variation before policy inference. This allows the policy to focus on task-relevant behavior rather than handling pose differences.}

\added{As shown in Fig.~\ref{fig:Pipeline}, given a point-cloud observation \( \mathbf{x} \in \mathcal{G}_i \) for which the canonical pose \(\mathbf{x}^{\mathrm{cn}}\) and the corresponding transformation have been obtained via the data canonicalization procedure, the canonical form is fed into the Point Cloud Aggregation Encoder Block (see the bottom of Fig.~\ref{fig:Pipeline}). Within this block, local-neighborhood features are extracted and processed through multiple fully connected layers to produce features for the canonicalized point cloud. The encoder design and hyperparameters are described in the next subsection.}

\added{
For robot state inputs and actions, we use end-effector positions and orientations, which are widely adopted in visual imitation learning~\cite{diffusion_policy,equidiff,DP3}.
In this setting, a data canonicalization pipeline analogous to that used for point clouds can be applied, aligning both the state and action to a canonical frame via the estimated $\mathrm{SE(3)}$ transformation.
It is worth noting that our method operates in Cartesian space rather than joint space representations (e.g., joint angles~\cite{zhao2024aloha}).}

\added{
In practice, the orientation component of both the robot state and action is not represented as a rotation matrix. The robot state typically encodes orientation using unit quaternions, which we convert to rotation matrices before applying the canonical transformation.
Similarly, the format of action orientation depends on the control mode: in absolute control, it is represented using the 6D continuous rotation formulation~\cite{zhou2019continuity}, while in relative control, it is specified as an axis-angle vector.
Both formats are first converted to rotation matrices before applying the transformation and converted back afterward for prediction.
For simplicity, we omit format conversions in the subsequent equations and use rotation matrices to represent $\mathrm{SO}(3)$ rotations.
Additionally, to streamline the derivation, we consider only a single robot state from the observation sequence \( \mathcal{O} \) of length \( m \), and a single action from the action sequence of length \( n \) in the subsequent discussion.}

\added{
More concretely, we represent the robot end-effector pose at timestep \( t \) using a homogeneous transformation matrix \( \mathbf{S}_t \in \mathrm{SE}(3) \), constructed from the position \( \mathbf{s}_t^{\mathrm{pos}} \in \mathbb{R}^{3 \times 1} \) and orientation \( \mathbf{s}_t^{\mathrm{ori}} \in \mathbb{R}^{3 \times 3} \).  
Let the canonical transformation matrix at timestep \( t \) be denoted as \( \mathbf{T}_t \in \mathrm{SE}(3) \), composed of a rotation \( \mathbf{R}_t \) and a translation vector \( \mathbf{x}_t^\mathrm{mn} \).  
The robot state is transformed via an $\mathrm{SE}(3)$ inverse operation to align it with the canonical point observation:
\[
\mathbf{S}_t^{\mathrm{cn}} = \mathbf{T}_t^{-1} \mathbf{S}_t.
\]
The gripper open width \( \mathbf{s}_t^{\mathrm{grip}} \), being invariant under rigid transformations, remains unchanged during this process.}

\added{
We now detail how canonical alignment is applied to the action representation under both absolute and relative control modes.}

\paragraph{\added{Absolute Control}} 
\added{
In the absolute control mode, the action represents the target end-effector pose in the world frame.
At diffusion step \( k \) and timestep \( t \), the action pose is represented by a homogeneous transformation matrix \( \mathbf{A}_{k,t} \in \mathrm{SE}(3) \), which is constructed from the position \( \mathbf{a}_{k,t}^{\mathrm{pos}} \in \mathbb{R}^{3 \times 1} \) and orientation \( \mathbf{a}_{k,t}^{\mathrm{ori}} \in \mathbb{R}^{3 \times 3} \).
To align it with the canonical observation frame, the following SE(3) inverse transformation is applied:}
\added{
\begin{equation}
\label{eq:action_inverse}
{\mathbf{A}}_{k,t}^{\mathrm{cn}} = \mathbf{T}_t^{-1} \mathbf{A}_{k,t}.
\end{equation}}

\added{
The canonicalized action \( {\mathbf{A}}_{k,t}^{\mathrm{cn}} \), together with the point cloud feature, robot state feature, and diffusion timestep feature, is provided as input to the Action Head for prediction (right side of Fig.~\ref{fig:Pipeline}). After denoising, the predicted canonical action \( {\mathbf{A}}_{k-1,t}^{\mathrm{cn}} \) is mapped back to the world frame using the forward SE(3) transformation:
\[
\mathbf{A}_{k-1,t} = \mathbf{T}_t {\mathbf{A}}_{k-1,t}^{\mathrm{cn}}.
\]}

\paragraph{\added{Relative Control}}
\added{
In the relative control mode, the action specifies the displacement of the end-effector pose relative to the previous timestep.  
Let the relative action pose at diffusion step \( k \) and timestep \( t \) be represented by a homogeneous transformation matrix \( \Delta \mathbf{A}_{k,t} \in \mathrm{SE}(3) \), such that}
\added{
\begin{equation}
\label{eq:relative_action}
\mathbf{A}_{k,t} = \Delta \mathbf{A}_{k,t} \cdot \mathbf{A}_{k,t-1}.
\end{equation}}

\added{
By substituting the relative action update in Equation~\eqref{eq:relative_action} into the canonical transformation (Equation~\eqref{eq:action_inverse}), we have:}
\begin{align}
\added{\mathbf{A}_{k,t}^\mathrm{cn}} &\added{= \mathbf{T}_t^{-1} \mathbf{A}_{k,t}} \notag \\
&\added{= \mathbf{T}_t^{-1} \Delta \mathbf{A}_{k,t} \cdot \mathbf{A}_{k,t-1}} \notag \\
&\added{= \left( \mathbf{T}_t^{-1} \Delta \mathbf{A}_{k,t} \mathbf{T}_t \right) 
          \cdot \left( \mathbf{T}_t^{-1} \mathbf{A}_{k,t-1} \right).} \label{eq:approx}
\end{align}

\added{
Since consecutive observations at high sampling rates, such as 10~Hz~\cite{wang2024dexcap,DP3}, typically exhibit only minor differences, particularly in smooth and continuous manipulation behaviors~\cite{diffusion_policy}, we approximate the canonical transformations at adjacent timesteps as equal, i.e., \( \mathbf{T}_{t-1} \approx \mathbf{T}_t \).
Accordingly, Equation~\eqref{eq:approx} becomes:}
\begin{align*}
\added{\mathbf{A}_{k,t}^\mathrm{cn}} &\added{= \left( \mathbf{T}_t^{-1} \Delta \mathbf{A}_{k,t} \mathbf{T}_t \right) \cdot \left( \mathbf{T}_t^{-1} \mathbf{A}_{k,t-1} \right)} \\
&\added{\approx \left( \mathbf{T}_t^{-1} \Delta \mathbf{A}_{k,t} \mathbf{T}_t \right) \cdot \left( \mathbf{T}_{t-1}^{-1} \mathbf{A}_{k,t-1} \right)} \\
&\added{= \left( \mathbf{T}_t^{-1} \Delta \mathbf{A}_{k,t} \mathbf{T}_t \right) \cdot \mathbf{A}_{k,t-1}^\mathrm{cn}}.
\end{align*}
\added{
Comparing this result with the canonical relative action update:}
\added{
\begin{equation*}
\mathbf{A}_{k,t}^\mathrm{cn} = \Delta \mathbf{A}_{k,t}^\mathrm{cn} \cdot \mathbf{A}_{k,t-1}^\mathrm{cn},
\end{equation*}}
\added{
we can identify the relative action pose expressed in the canonical frame as}
\added{
\[
\Delta \mathbf{A}_{k,t}^\mathrm{cn} = \mathbf{T}_t^{-1} \Delta \mathbf{A}_{k,t} \mathbf{T}_t.
\]}

\added{
The canonicalized relative action \( \Delta \mathbf{A}_{k,t}^\mathrm{cn} \) is used as input to the Action Head.  
After prediction, the output is transformed back to the world frame using the corresponding $\mathrm{SE}(3)$ forward transformation:}
\added{
\[
\Delta \mathbf{A}_{k-1,t} = \mathbf{T}_t \Delta \mathbf{A}_{k-1,t}^\mathrm{cn} \mathbf{T}_t^{-1}.
\]}

\added{
The gripper command \( \mathbf{a}_{k,t}^{\mathrm{grip}} \), being invariant to rigid-body transformations, remains unchanged in both absolute and relative control modes.}

\added{
This process is repeated until $k=0$. During training, the estimated action is constrained by its ground-truth counterpart. During inference, the same pipeline is followed, with the final output serving as the predicted trajectory.}

\added{
By mapping all training data to their canonical distribution, we ensure that the diffusion model learns representations invariant to different transformations within the equivariant group. This enables the model to generalize to unseen poses during evaluation, as each input is first transformed into its corresponding canonical pose, regardless of its original pose. As a result, the model can infer the correct action since it has already seen the canonicalized representation during training.}

\subsection{Point Cloud Aggregation Encoder}
\added{Having established the canonicalized observations in the previous subsection,}
we \added{now} introduce a point cloud aggregation encoder designed to extract informative features by leveraging the spatial structure inherent to point cloud inputs in the canonical frame. This facilitates robust policy learning from geometrically consistent inputs.

DP3 \cite{DP3} explores various point cloud encoders that have demonstrated effectiveness in computer vision tasks, including PointNet++ \cite{pointnet++}, PointNeXt \cite{pointnext}, and Point Transformer \cite{pointtransformer}. Their findings suggest that a simple multi-layer perceptron (MLP) is a highly effective approach for point cloud encoding in imitation learning. In addition to MLPs, we also examine alternative encoders such as DGCNN \cite{DGCNN}, with a detailed comparison presented in the following section. Based on these insights, we adopt a simple MLP architecture with residual connections, facilitating the training of deeper networks for robust feature extraction.

In image feature extraction, convolution is particularly effective as it captures both pixel-wise information and contextual dependencies from neighboring pixels. Similarly, in point cloud analysis, incorporating neighboring information enhances feature robustness, mitigating the impact of noise. While individual points may be noisy, their neighbors often offer stable contextual cues that help preserve meaningful structure.

To exploit this property, we first identify the neighbors of each point in the point cloud before applying MLP-based feature extraction. PointMLP \cite{PointMLP} introduces a geometric affine module that effectively aggregates information from neighboring points, as illustrated in Fig. \ref{fig:Affine}.

\begin{figure}[t] \centering \includegraphics[width=1.0\linewidth]{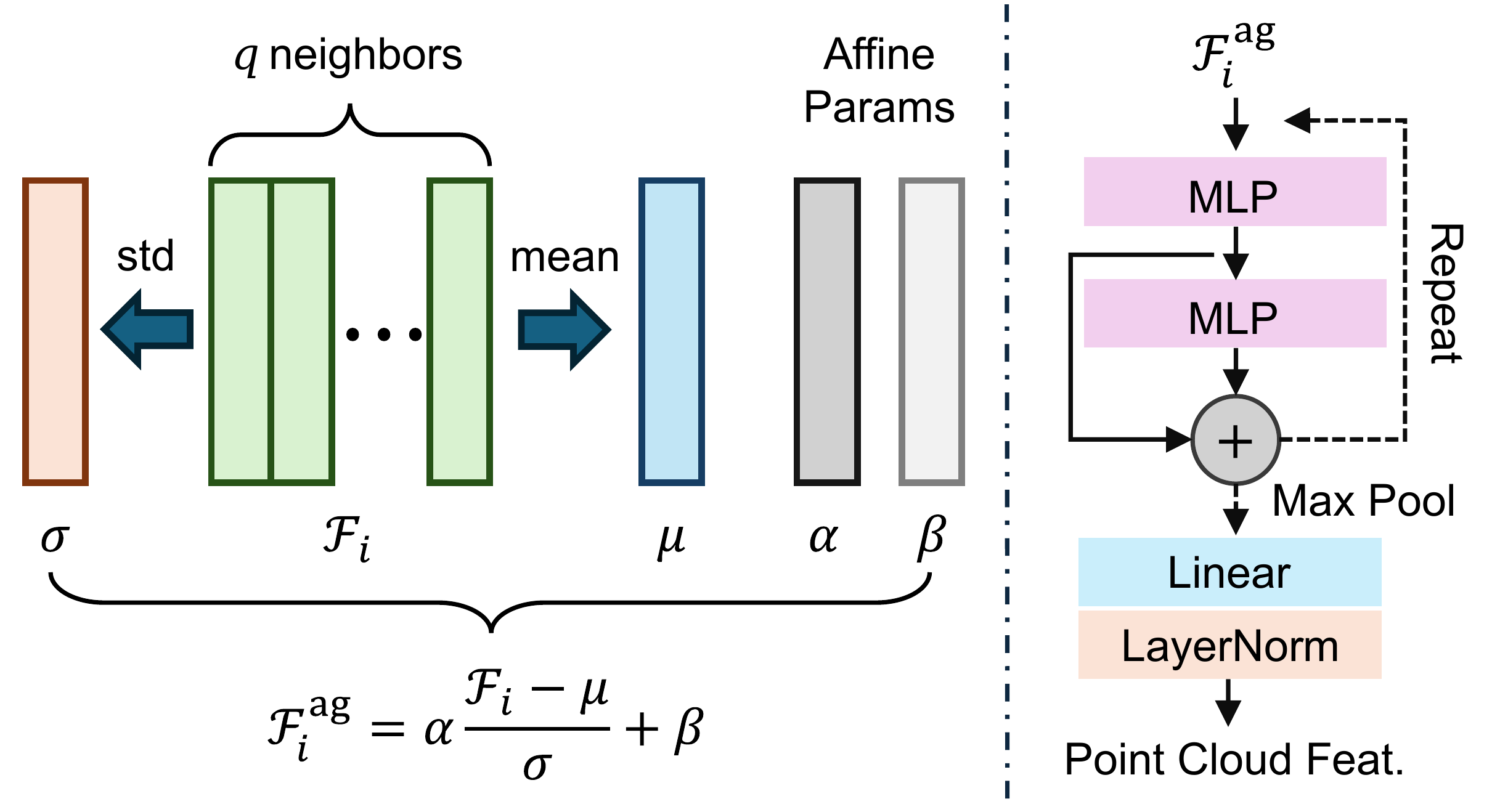} \caption{Neighborhood aggregation module. Point features are normalized and affine-transformed based on local neighbors. A residual MLP with max pooling is used to extract permutation-invariant point cloud features.} \label{fig:Affine}
\end{figure}

First, the neighborhood of each point is determined based on its $\mathit{xyz}$ coordinates, yielding a set of neighboring features represented as $\mathcal{F}_i=\{\mathbf{F}_i^1, \mathbf{F}_i^2, \dots, \mathbf{F}_i^q\}$, where $i$ denotes the $i$th point and $q$ represents the number of neighbors. Within this local neighborhood, we compute the mean $\mu$ and standard deviation $\sigma$ of the features and apply feature normalization. To further enhance robustness in feature aggregation, affine transformation parameters $\alpha$ and $\beta$ are applied to the normalized local features, producing the final aggregated features $\mathcal{F}_i^\mathrm{ag}$: 
\begin{equation*}
\label{Eq:affine}
\mathcal{F}_i^\mathrm{ag}=\alpha \frac{\mathcal{F}_i-\mu}{\sigma}+\beta,
\end{equation*}
which are subsequently fed into the residual MLP encoder for processing.
In the remaining sections, we fix $q = 10$ unless otherwise specified. The parameters $\alpha$ and $\beta$ are learned by the network.

\section{Simulation Experiments}
\label{sec:sim}
In this section, we provide a comprehensive evaluation of canonical policy in simulation environments. We first present an experiment to validate the $\mathrm{SE}(3)$ equivariance property of canonical policy. We then conduct benchmark comparisons with several representative point cloud-based baselines across a diverse set of manipulation tasks.
\added{We further perform ablation studies to highlight the contributions of key components in our method.
Finally, we present a dedicated analysis of the reliability of the rotation estimator introduced in Section~\ref{sec:preprocessing}, examining its performance under noisy inputs and illustrating the results with feature visualizations.}

\subsection{Experiment Setup}
We evaluate canonical policy across a diverse set of simulated manipulation tasks. To assess generalization and performance, we compare it against several representative point cloud-based policies under consistent settings.

\textbf{DP3}~\cite{DP3} is a diffusion-based policy that directly processes 3D point clouds using a simple MLP encoder. 

\textbf{iDP3}~\cite{iDP3} builds upon DP3 by replacing the MLP encoder with a multi-layer 1D convolutional encoder. This architecture is designed to extract multi-resolution features from point clouds, capturing more detailed and hierarchical point cloud information.

\textbf{EquiBot}~\cite{Equibot} is a \added{$\mathrm{SIM}(3)$}-equivariant policy that uses the same equivariant encoder as canonical policy, built on the Vector Neuron framework \cite{VN}. It modifies the diffusion U-Net backbone using equivariant linear, pooling, and normalization layers. To conform with the Vector Neuron framework, EquiBot reformulates low-dimensional states and actions into vector and scalar representations.

\begin{figure*}[t]
    \centering
    \includegraphics[width=\textwidth]{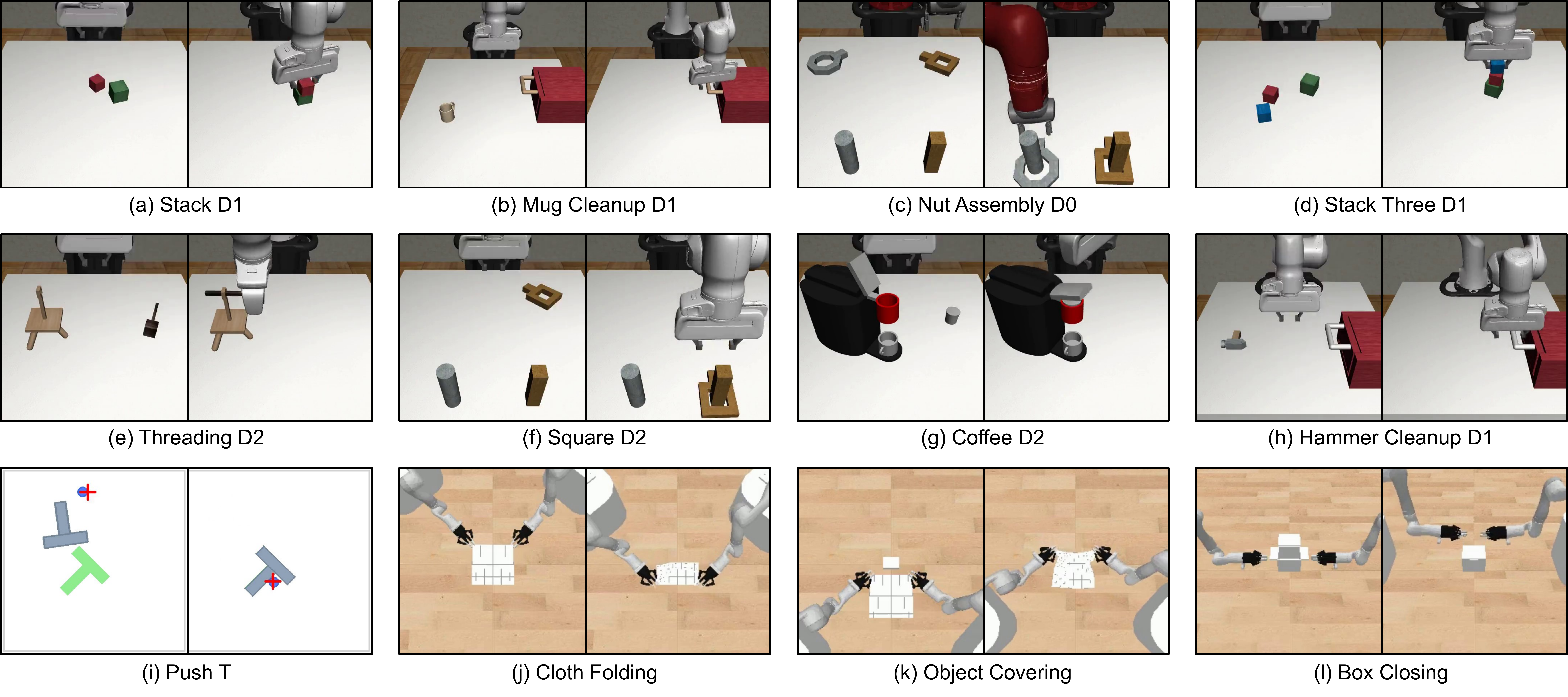}
    \caption{Visualization of the 12 simulated manipulation tasks used in the simulation benchmark.
    Subfigures (a)–(h) are from MimicGen~\cite{robomimic}, (i) is from Diffusion Policy~\cite{diffusion_policy}, and (j)–(l) are from EquiBot~\cite{Equibot}.
    In each subfigure, the left image shows the initial state, and the right image shows the goal state of the task.}
    \label{fig:sim_tasks}
\end{figure*}
We evaluate all policies on three categories of tasks, each differing in task types and complexity:

\textbf{MimicGen Tasks}~\cite{robomimic} are visualized in Fig.~\ref{fig:sim_tasks}(a)--(h). These tasks, such as coffee making and table organization, simulate realistic household manipulation scenarios. Point cloud observations are obtained from four camera views. To ensure computational efficiency without compromising performance, all point clouds are uniformly downsampled to 256 points using farthest point sampling~\cite{pointnet++}.

\textbf{Push T Task} is shown in Fig.~\ref{fig:sim_tasks}(i). Originally a 2D task from Diffusion Policy~\cite{diffusion_policy}, we adapt it to 3D by extending each of the 9 keypoints with a zero z-coordinate, forming a sparse 3D point cloud. The goal is to push a T-shaped block to a placed T-shaped target.

\textbf{Bimanual Tasks} are depicted in Fig.~\ref{fig:sim_tasks}(j)--(l). These dual-arm manipulation tasks are adapted from the released EquiBot dataset~\cite{Equibot}, which provides 1024 point clouds as observation.

For all simulation tasks, the input to all policies consists of point clouds with spatial coordinates only, excluding any color information.

In canonical policy, we employ \textit{absolute control} for the MimicGen and Push T tasks, and \textit{relative control} for the bimanual tasks.
The difference between the two control modes lies in the choice of action space used during the SE(3) inverse and forward transformations. See Section~\ref{sec:overview} for more details.

While all baselines in this section are based on diffusion models for action prediction, canonical policy framework is compatible with alternative action heads, such as flow matching. We  explore this flexibility in the ablation study.

\subsection{Validation of SE(3) Equivariance}
\begin{figure*}[t]
    \centering
    \includegraphics[width=1.0\linewidth]{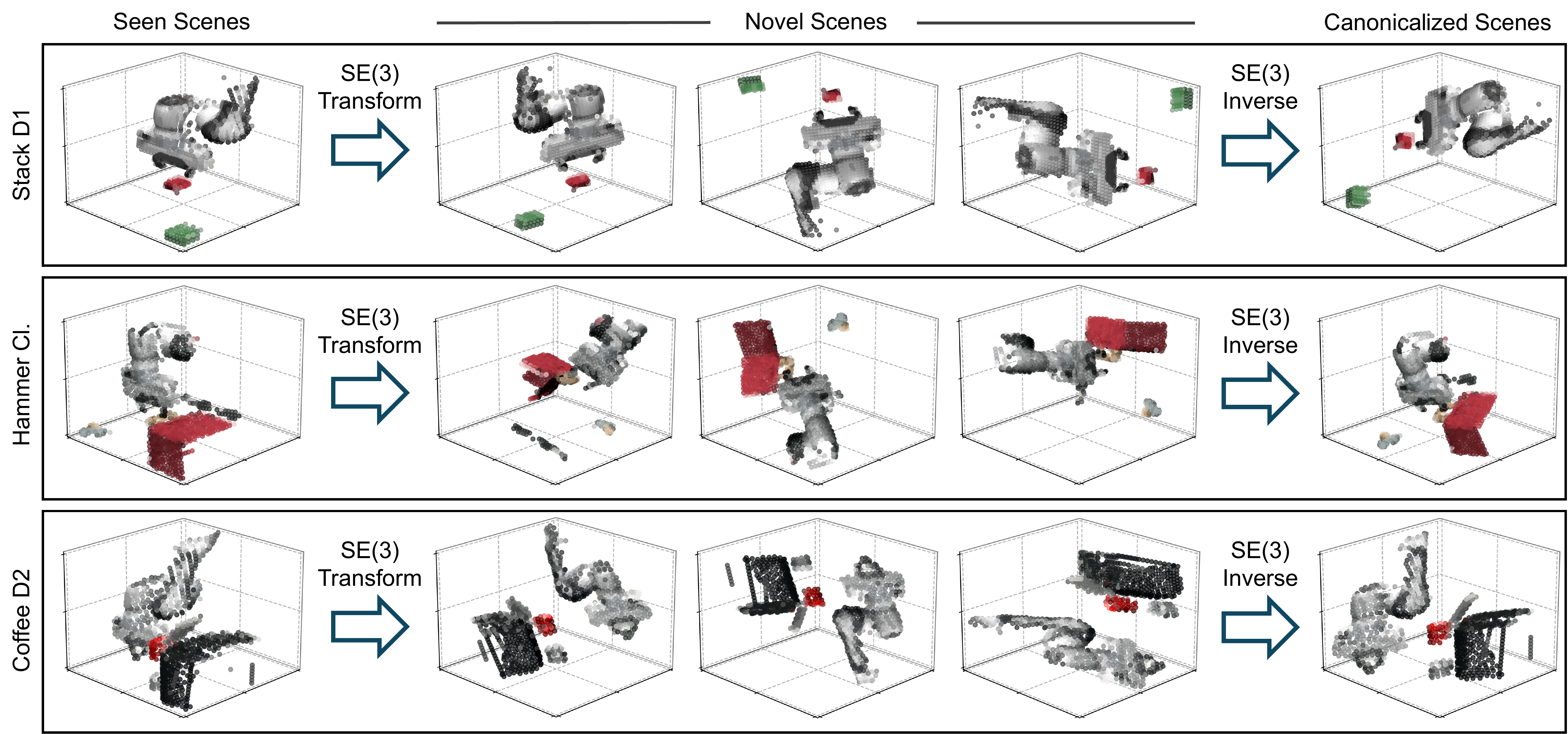}
    \caption{Canonical pose visualization across tasks. \added{Seen} and \added{novel scenes} are aligned to a shared canonical frame via estimated SE(3) inverse transforms.}
    \label{fig:vis_equi}
\end{figure*}

\begin{figure*}[t]
    \centering
    \includegraphics[width=1.0\linewidth]{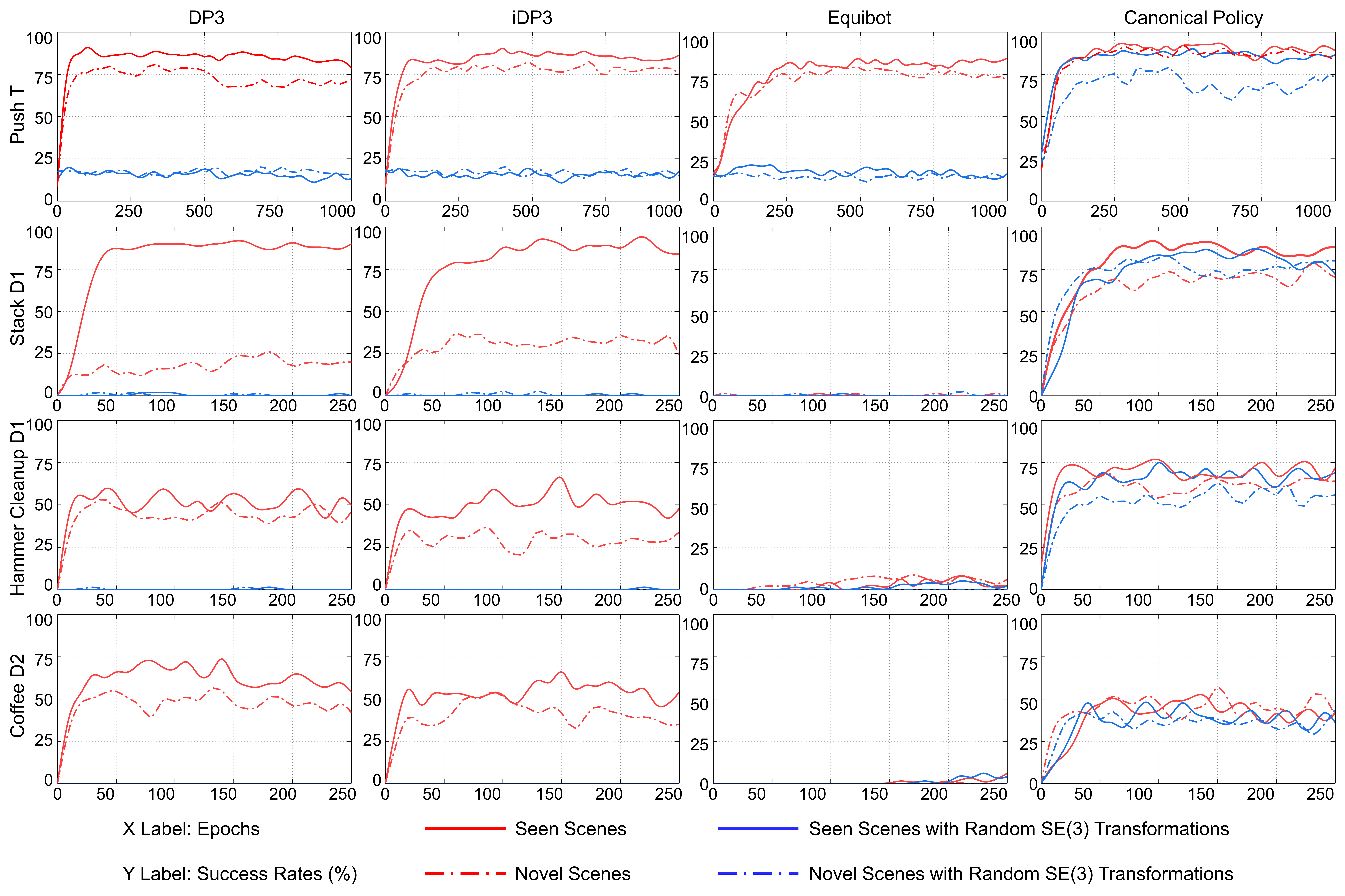}
    \caption{\added{Quantitative evaluation of different policies across tasks. Curves show success rates over training epochs for seen scenes and novel scenes, with and without random $\mathrm{SE}(3)$ transformations. Canonical Policy maintains stable performance under both seen and novel conditions, while baseline policies degrade under novel conditions.}}
    \label{fig:vis_se3_plot}
\end{figure*}

Before comparing the performance of different policies across various tasks, we first validate the theoretical foundation of canonical policy introduced in Section~\ref{sec:method}.

As illustrated in Fig.~\ref{fig:vis_equi}, we visualize three MimicGen tasks: Stack D1, Hammer Cleanup D1, and Coffee D2, to demonstrate the SE(3)-equivariance property of canonical policy. As discussed in Section~\ref{sec:theory}, the left column of Fig.~\ref{fig:vis_equi} presents point cloud observations from the training \added{dataset}, which are considered \added{seen scenes} and denoted as~$\mathcal{D}$. Next, random SE(3) transformations are applied to generate \added{novel scenes} $\mathcal{D}^\mathrm{rd}$, shown in the middle column. Finally, by applying data canonicalization pipeline described in Section~\ref{sec:preprocessing}, \added{we first center each transformed point cloud by subtracting its mean, 
then estimate the associated rotation using an \(\mathrm{SO}(3)\)-equivariant network (Equation~\eqref{eq:rotation_matrix}), 
and finally apply the inverse rotation to each point cloud, aligning it to a shared canonical frame}, as visualized in the right column of Fig.~\ref{fig:vis_equi}.

\added{To quantitatively assess the $\mathrm{SE}(3)$ equivariance of Canonical Policy, we design an experiment under four conditions, as shown in Fig.~\ref{fig:vis_se3_plot}: (i) seen scenes (training set; solid red), (ii) seen scenes with random $\mathrm{SE}(3)$ transforms (solid blue), (iii) novel scenes (evaluation set; dashed red), and (iv) novel scenes with random $\mathrm{SE}(3)$ transforms (dashed blue). All policies are trained on the training set and tested under all four conditions.}

\added{For the Canonical Policy, as described in Section~\ref{sec:overview}, both the observation and the low-dimensional inputs are mapped to a canonical frame using the estimated transformation; actions are predicted in this frame and then mapped back using the inverse transform. Consequently, samples that differ only by $\mathrm{SE}(3)$ transforms but belong to the same equivariant group share the same canonical representation. Hence, as shown in the last column of Fig.~\ref{fig:vis_se3_plot}, training on seen data and testing on either seen or novel data, with or without random $\mathrm{SE}(3)$ transforms, yields similar success rate curves.}

\added{In contrast, baseline policies that do not enforce $\mathrm{SE}(3)$ equivariance generalize poorly: performance drops on novel scenes and degrades sharply, often to near zero, when random $\mathrm{SE}(3)$ transforms are introduced at evaluation. These findings verify the theoretical motivation in Section~\ref{sec:method}: canonicalization removes pose variability before policy inference, leading to robustness across arbitrary rigid transforms.}

\subsection{Benchmark with Point Cloud Policies}
As discussed in Section~\ref{sec:preprocessing}, we provide two versions of canonical policy: one with translation invariance and full $\mathrm{SO}(3)$ equivariance, referred to as \added{CP-$\mathrm{SO}3$}; and another with translation invariance and in-plane $\mathrm{SO}(2)$ equivariance, referred to as \added{CP-$\mathrm{SO}2$}. We will use these abbreviations throughout the following experimental section.

Table~\ref{tab:comparison} summarizes the performance of various policies across 12 tasks. Specifically, for the 8 MimicGen tasks, we generate 200 demonstrations~\cite{equidiff} and train each method for 250 epochs. For the Push T task, which also contains 200 demonstrations~\cite{diffusion_policy}, each method is trained for 1000 epochs. For the bimanual tasks, each with 50 demonstrations~\cite{Equibot}, training is conducted for 2000 epochs per method.

All experiments are conducted using three random environment seeds: 42, 43, and 44. For each seed, we evaluate the policies over the last 10 epochs and report the average. For MimicGen and Push T, each rollout is conducted across 50 different environment initializations. For the three bimanual tasks, each rollout consists of 10 single trials. Thus, the mean score reported in Table~\ref{tab:comparison} for each task and each method is the average over 30 evaluation results. The best result for each task is shown in bold, and the second-best is underlined.

As shown in Table~\ref{tab:comparison}, both CP-SO3 and CP-SO2 outperform the baseline point cloud policies (DP3, iDP3, and EquiBot) on all 8 MimicGen tasks and the 3D Push T task. Notably, CP-SO2 achieves the highest success rates in 6 out of the 9 single-arm tasks, supporting our hypothesis that $\mathrm{SO}(2)$ equivariance is particularly well-suited for scenarios where the $z$-axis remains consistently upward.
In the three bimanual tasks, all policies perform similarly in Object Covering and Box Closing, achieving success rates close to 100\%. However, for Cloth Folding, CP-SO3 and CP-SO2 achieve the second and third best performances, respectively, following EquiBot.

When averaged across all tasks, CP-SO2 attains the highest overall performance with a success rate of 58.3\%, followed by CP-SO3 at 56.0\%. Compared to the baselines, CP-SO2 outperforms DP3 by 13.4\%, iDP3 by 15.4\%, and EquiBot by 25.1\%.

Interestingly, while EquiBot achieves the best performance across all three bimanual tasks and ranks third on the 3D Push T task, it performs poorly on the MimicGen tasks. We hypothesize that this discrepancy arises from the point cloud encoder used in EquiBot \cite{DP3}, which will further explored in our ablation study on point cloud encoders.

\begin{table*}[htbp]
\caption{Task success rates (\%) across different policies and tasks.}
\label{tab:comparison}
\centering
\renewcommand{\arraystretch}{1.2} 

\newcolumntype{C}{>{\centering\arraybackslash}X}

\begin{tabularx}{\textwidth}{lCCCCCC}
\toprule
 & \textbf{Stack D1} & \textbf{Mug Cleanup D1} & \textbf{Nut Assembly D0} & \textbf{Stack Three D1} & \textbf{Threading D2} & \textbf{Square D2} \\
\midrule
DP3 & $23\pm3$ & $28\pm6$ & $21\pm6$ & $1\pm1$ & $8\pm3$ & $3\pm3$ \\
iDP3 & $23\pm7$ & $25\pm7$ & $9\pm5$ & $0\pm1$ & $6\pm3$ & $1\pm2$ \\
EquiBot & $1\pm1$ & $5\pm3$ & $0\pm0$ & $0\pm0$ & $0\pm1$ & $0\pm1$ \\
\rowcolor{gray!30}
CP-SO3 & \underline{$71\pm5$} & \underline{$37\pm8$} & \underline{$39\pm11$} & \underline{$8\pm3$} & \bm{$17\pm5$} & \bm{$16\pm4$} \\
\rowcolor{gray!30}
CP-SO2 & \bm{$79\pm7$} & \bm{$38\pm8$} & \bm{$42\pm7$} & \bm{$9\pm3$} & \underline{$15\pm4$} & \underline{$15\pm5$} \\
\bottomrule
\end{tabularx}

\vspace{1.5em}

\begin{tabularx}{\textwidth}{lCCCCCC}
\toprule
 & \textbf{Coffee D2} & \textbf{Hammer Cl. D1} & \textbf{Push T} & \textbf{Cloth Folding} & \textbf{Object Covering} & \textbf{Box Closing} \\
\midrule
DP3 & \underline{$43\pm6$} & $54\pm7$ & $71\pm4$ & $88\pm7$ & \bm{$99\pm1$} & \bm{$100\pm0$} \\
iDP3 & $40\pm6$ & $45\pm7$ & $76\pm1$ & $91\pm6$ & \bm{$99\pm1$} & \bm{$100\pm0$} \\
EquiBot & $1\pm1$ & $5\pm2$ & $77\pm4$ & \bm{$98\pm2$} & \bm{$99\pm1$} & \bm{$100\pm0$} \\
\rowcolor{gray!30}
CP-SO3 & $37\pm6$ & \bm{$68\pm6$} & \underline{$84\pm3$} & \underline{$97\pm6$} & \bm{$99\pm1$} & $99\pm6$ \\
\rowcolor{gray!30}
CP-SO2 & \bm{$51\pm6$} & \underline{$67\pm6$} & \bm{$89\pm3$} & $95\pm6$ & \bm{$99\pm1$} & \bm{$100\pm0$} \\
\bottomrule
\end{tabularx}

\vspace{0.3em}
\footnotesize
\parbox{1.0\linewidth}{
We report the mean success rate, computed as the average over the final 10 checkpoints, each evaluated in 50 environment initializations. Results are further averaged over 3 training seeds, yielding 150 evaluations per policy. \textbf{Bold} indicates the best performance; \underline{underline} denotes the second-best. ``CP-SO3'' refers to canonical policy with SO(3) rotation equivariance, and ``CP-SO2'' denotes the variant with SO(2) equivariance. \added{All policies are evaluated on point cloud inputs with only 3D coordinates, without using color information.}
}
\end{table*}

\subsection{Ablation Study}
\added{We conducted comprehensive ablation studies focusing on five aspects: 
(i) the contributions of different modules within the canonical policy, 
(ii) the effects of various point cloud encoders,
(iii) the role of scaling in task performance via comparison of $\mathrm{SE}(3)$ and $\mathrm{SIM}(3)$, 
(iv) the impact of key parameters in the $\mathrm{SO}(3)$-equivariant network on estimation performance, 
and (v) the influence of different action heads.
}
In all experiments, success rates are computed as the mean over the final 10 checkpoints, each evaluated under 50 environment initializations with random seed 42.

As shown in Table~\ref{tab:ab_cp}, the abbreviations “Aggre.”, “Trans.”, “SO2”, and “SO3” represent the Point Cloud Aggregation Encoder Block, translation invariance, SO(2) equivariance, and SO(3) equivariance, respectively. The results show that, compared to the first row which corresponds to the original DP3 setup, incorporating the Aggregation Encoder Block together with the canonical representation effectively facilitates the learning of 3D policies.

In the point cloud encoder ablation study (Table~\ref{tab:ab_encoder}), we compared three encoders: the original MLP-based DP3 encoder, PointNet++ \cite{pointnet++}, and a graph-based encoder DGCNN \cite{DGCNN}. These encoders were evaluated under different policies, including DP3, EquiBot, CP-SO3, and CP-SO2. To explain EquiBot’s limited performance on RoboMimic tasks, we also included its original encoder, the Vector Neuron SO(3)-equivariant encoder. Across all policies, when using an encoder other than the DP3 encoder, the success rate in MimicGen tasks dropped to nearly zero. Similarly, in the Push T task, the success rate remained notably limited.

Interestingly, while EquiBot performs poorly with its original SO(3)-equivariant encoder, replacing it with the DP3 encoder leads to substantial performance improvements. On MimicGen tasks, this change yields significantly higher success rates, and on the Push T task, EquiBot achieves a 90\% success rate, which is the highest among all evaluated policies.

\added{
Furthermore, EquiBot adopts \(\mathrm{SIM}(3)\) equivariance, which extends \(\mathrm{SE}(3)\) equivariance by additionally scaling the entire scene to improve robustness to variations in object size. To evaluate this effect within our canonical policy, we conduct an ablation study comparing policies with and without scaling across four MimicGen tasks, as shown in Table~\ref{tab:SIM3}.
Results are mixed: for both CP-SO3 and CP-SO2, scaling improves performance on two tasks but worsens it on the other two, leading us to adopt $\mathrm{SE}(3)$ canonicalization in the final model.
}

\added{Next, we analyze two key hyperparameters of the \(\mathrm{SO}(3)\)-equivariant network in Fig.~\ref{fig:SO3} that mainly determine the parameter budget of the entire policy: the number of repeat layers and the feature dimension of the equivariant linear layers. 
We evaluate their impact on the Stack D1 task for both CP-SO3 and CP-SO2, with results summarized in Table~\ref{tab:params}. 
Table~\ref{tab:params}(a) reports the proportion of parameters allocated to the \(\mathrm{SO}(3)\)-equivariant branch relative to the full policy (including the \(\mathrm{SO}(3)\)-equivariant branch, the point cloud encoder, and the diffusion action head). 
Two main observations emerge. 
First, the \(\mathrm{SO}(3)\)-equivariant branch is lightweight: under our default configuration (4 repeat layers, 48 feature dimensions) it accounts for only \(\sim\)3.2\% of total parameters. 
Second, feature dimension has a stronger effect on parameter count than repeat depth, yielding a greater monotonic increase as the dimension grows. 
Performance trends in Table~\ref{tab:params}(b) (CP-SO3) and Table~\ref{tab:params}(c) (CP-SO2) further show that success rates improve with larger branch capacity, consistent with the intuition that a higher-capacity equivariant estimator is more robust to noisy inputs. This will be discussed in detail in the next subsection.}

Finally, to examine canonical policy’s compatibility with different action heads, we compared diffusion model and flow matching~\cite{flow_matching} on Stack D1 and Push T. The results are presented in Table~\ref{tab:ab_policy}. It is noteworthy that diffusion and flow matching exhibit task-specific preferences. For instance, in Stack D1, diffusion outperforms flow matching for nearly all policies, whereas in Push T, flow matching leads to better results. Despite these differences, canonical policy consistently outperforms the other baselines under both types of action heads.

\newcolumntype{C}[1]{>{\centering\arraybackslash}p{#1}}
\begin{table*}[htbp]
\caption{Task success rates (\%) in the ablation study across different configurations.}
\label{tab:ab_cp}
\centering
\renewcommand{\arraystretch}{1.2} 
\begin{tabular}{C{0.6cm}C{0.6cm}C{0.6cm}C{0.6cm}C{1.4cm}C{2.2cm}C{2.2cm}C{2.0cm}C{1.4cm}C{1.0cm}}
\toprule
\textbf{Aggre.} & \textbf{Trans.} & \textbf{SO2} & \textbf{SO3} & \textbf{Stack D1} & \textbf{Mug Cleanup D1} & \textbf{Nut Assembly D0} & \textbf{Hammer Cl. D1} & \textbf{Coffee D2} & \textbf{Push T} \\
\midrule
$\times$ & $\times$ & $\times$ & $\times$ & $20\pm2$ & $25\pm6$ & $21\pm1$ & $54\pm6$ & $44\pm5$ & $70\pm2$ \\
$\checkmark$ & $\times$ & $\times$ & $\times$ & $32\pm5$ & $28\pm8$ & $12\pm2$ & $62\pm5$ & $37\pm7$ & $69\pm2$ \\
$\times$ & $\checkmark$ & $\times$ & $\times$ & $43\pm3$ & $35\pm4$ & $34\pm4$ & $59\pm6$ & $49\pm5$ & $86\pm1$ \\
$\checkmark$ & $\checkmark$ & $\times$ & $\times$ & \underline{$75\pm5$} & $34\pm8$ & $14\pm3$ & $66\pm6$ & $38\pm5$ & \bm{$88\pm2$} \\
$\times$ & $\checkmark$ & $\checkmark$ & $\times$ & $57\pm3$ & \underline{$37\pm4$} & $30\pm4$ & \underline{$67\pm6$} & \underline{$51\pm4$} & $84\pm2$ \\
$\times$ & $\checkmark$ & $\times$ & $\checkmark$ & $50\pm5$ & $36\pm6$ & $29\pm2$ & $57\pm8$ & $39\pm5$ & $82\pm3$ \\
\rowcolor{gray!30}
$\checkmark$ & $\checkmark$ & $\times$ & $\checkmark$ & $68\pm4$ & \underline{$37\pm5$} & \bm{$41\pm4$} & \bm{$68\pm5$} & $40\pm7$ & $85\pm2$ \\
\rowcolor{gray!30}
$\checkmark$ & $\checkmark$ & $\checkmark$ & $\times$ & \bm{$76\pm5$} & \bm{$44\pm7$} & \underline{$38\pm6$} & \underline{$67\pm8$} & \bm{$52\pm8$} & \underline{$87\pm3$} \\
\bottomrule
\end{tabular}

\vspace{0.3em}
\footnotesize
\parbox{0.93\linewidth}{
Mean success rates are averaged over the final 10 checkpoints, each evaluated in 50 environment initializations with random seed 42. “Aggre.” refers to the Point Cloud Aggregation Encoder Block, and “Trans.” denotes translation invariance.
}
\end{table*}

\begin{table*}[htbp]
\caption{Task success rates (\%) for different encoders across policies and tasks.}
\label{tab:ab_encoder}
\centering
\renewcommand{\arraystretch}{1.2}
\setlength{\tabcolsep}{3pt}
\newcolumntype{L}[1]{>{\raggedright\arraybackslash}p{#1}} %
\newcolumntype{C}[1]{>{\centering\arraybackslash}p{#1}}    %

\begin{tabular}{L{1.4cm} L{2.0cm} C{1.6cm} C{2.3cm} C{2.3cm} C{2.3cm} C{1.6cm} C{1.6cm}}
\toprule
\multicolumn{2}{c}{} & \textbf{Stack D1} & \textbf{Mug Cleanup D1} & \textbf{Nut Assembly D0} & \textbf{Hammer Cl. D1} & \textbf{Coffee D2} & \textbf{Push T} \\
\midrule

DP3 & DP3 Encoder & $20\pm2$ & $25\pm6$ & $21\pm1$ & $54\pm6$ & \underline{$44\pm5$} & $70\pm2$ \\
    & PointNet++ & $0\pm0$ & $1\pm1$ & $0\pm0$ & $0\pm1$ & $0\pm0$ & $13\pm1$ \\
    & DGCNN & $0\pm0$ & $0\pm0$ & $0\pm0$ & $0\pm0$ & $0\pm0$ & $10\pm0$ \\

\midrule
Equibot & DP3 Encoder & $37\pm5$ & $32\pm9$ & $20\pm3$ & \underline{$58\pm8$} & $40\pm6$ & \bm{$90\pm0$} \\
        & Equibot Encoder & $1\pm1$ & $3\pm3$ & $0\pm0$ & $5\pm3$ & $1\pm1$ & $74\pm3$ \\
        & PointNet++ & $0\pm0$ & $0\pm0$ & $0\pm0$ & $1\pm1$ & $0\pm0$ & $13\pm1$ \\
        & DGCNN & $0\pm1$ & $0\pm0$ & $0\pm0$ & $0\pm0$ & $0\pm0$ & $22\pm1$ \\

\midrule
\rowcolor{gray!30}
CP-SO3 & DP3 Encoder & \underline{$50\pm5$} & \underline{$36\pm6$} & \underline{$29\pm2$} & $57\pm8$ & $39\pm5$ & $82\pm3$ \\
\rowcolor{gray!30}
       & PointNet++ & $0\pm0$ & $0\pm0$ & $0\pm0$ & $1\pm1$ & $0\pm1$ & $28\pm3$ \\
\rowcolor{gray!30}
       & DGCNN & $0\pm0$ & $0\pm0$ & $0\pm0$ & $0\pm1$ & $0\pm0$ & $17\pm1$ \\

\midrule
\rowcolor{gray!30}
CP-SO2 & DP3 Encoder & \bm{$57\pm3$} & \bm{$37\pm4$} & \bm{$30\pm4$} & \bm{$67\pm6$} & \bm{$51\pm4$} & \underline{$84\pm2$} \\
\rowcolor{gray!30}
       & PointNet++ & $0\pm1$ & $0\pm0$ & $0\pm0$ & $3\pm3$ & $1\pm1$ & $32\pm1$ \\
\rowcolor{gray!30}
       & DGCNN & $0\pm0$ & $0\pm0$ & $0\pm0$ & $0\pm0$ & $0\pm0$ & $27\pm1$ \\

\bottomrule
\end{tabular}

\vspace{0.3em}
\footnotesize
\parbox{0.93\linewidth}{
Mean success rates are averaged over the final 10 checkpoints, each evaluated in 50 environment initializations with random seed 42.
}
\end{table*}

\begin{table}[t]
\centering
\caption{\added{Ablation of scaling ($\mathrm{SIM}(3)$) versus no scaling ($\mathrm{SE}(3)$) on canonical policy performance across tasks.}}
\label{tab:SIM3}
\small
\setlength{\tabcolsep}{3.6pt}
\renewcommand{\arraystretch}{1.15}
\begin{tabular}{>{\centering\arraybackslash}p{1.0cm}  
                >{\centering\arraybackslash}p{0.9cm}  
                >{\centering\arraybackslash}p{1.3cm}
                >{\centering\arraybackslash}p{1.3cm}
                >{\centering\arraybackslash}p{1.3cm}
                >{\centering\arraybackslash}p{1.4cm}}
\toprule
 & Scaling &
\textbf{Stack D1} & \textbf{Mug. D1} &
\textbf{Ham. D1} & \textbf{Coffee D2} \\
\midrule
\multirow{2}{*}{CP-SO3}
 & $\times$ & $68\pm4$ & $37\pm5$ & \bm{$68\pm5$} & \bm{$40\pm7$} \\
 & \checkmark  & \bm{$80\pm5$} & \bm{$41\pm5$} & $65\pm6$ & $39\pm3$ \\
\midrule
\multirow{2}{*}{CP-SO2}
 & $\times$ & \bm{$76\pm5$} & $44\pm7$ & $67\pm8$ & \bm{$52\pm8$} \\
 & \checkmark  & $69\pm4$ & \bm{$47\pm7$} & \bm{$71\pm5$} & $47\pm5$ \\
\bottomrule
\end{tabular}
\vspace{0.3em}
\footnotesize
\parbox{0.93\linewidth}{\added{
Mean success rates are averaged over the final 10 checkpoints, each evaluated in 50 environment initializations with random seed 42.}}
\end{table}

\begin{table}[t]
\caption{\added{Ablation on two key hyperparameters of the $\mathrm{SO}(3)$-equivariant network and performance on Stack D1.}}
\label{tab:params}
\small
\setlength{\tabcolsep}{9pt}     
\renewcommand{\arraystretch}{1.18}

\centering
\begin{tabular}{lcccc}          
\toprule
\multirow{2}{*}{\makecell[l]{\textbf{Feature}\\\textbf{Dimensions}}} &
\multicolumn{3}{c}{\textbf{Repeat Layers}} &
\multirow{2}{*}{\textbf{Average}} \\
\cmidrule(lr){2-4}
& \textbf{1} & \textbf{2} & \textbf{4} & \\
\midrule
12 & 0.20\% & 0.21\% & 0.22\% & 0.21\% \\
24 & 0.80\% & 0.82\% & 0.87\% & \underline{0.83}\% \\
48 & 3.11\% & 3.20\% & 3.39\% & \textbf{3.23}\% \\
\bottomrule
\end{tabular}

\medskip
{\footnotesize\emph{\added{(a) Proportion of $\mathrm{SO}(3)$-equivariant branch parameters relative to total policy parameters.}}}

\bigskip

\begin{tabular}{lcccc}
\toprule
\multirow{2}{*}{\makecell[l]{\textbf{Feature}\\\textbf{Dimensions}}} &
\multicolumn{3}{c}{\textbf{Repeat Layers}} &
\multirow{2}{*}{\textbf{Average}} \\
\cmidrule(lr){2-4}
& \textbf{1} & \textbf{2} & \textbf{4} & \\
\midrule
12 & $68\!\pm\!3$ & $62\!\pm\!4$ & $61\!\pm\!4$ & $64\!\pm\!5$ \\
24 & $73\!\pm\!5$ & $63\!\pm\!4$ & \textbf{$77\!\pm\!4$} & \bm{$71\!\pm\!7$} \\
48 & $74\!\pm\!4$ & $70\!\pm\!4$ & $68\!\pm\!4$ & \underline{$70\!\pm\!5$} \\
\bottomrule
\end{tabular}

\medskip
{\footnotesize\emph{\added{(b) Success rate ($\%$) of CP-SO3 with varying branch size.}}}

\bigskip

\centering
\begin{tabular}{lcccc}
\toprule
\multirow{2}{*}{\makecell[l]{\textbf{Feature}\\\textbf{Dimensions}}} &
\multicolumn{3}{c}{\textbf{Repeat Layers}} &
\multirow{2}{*}{\textbf{Average}} \\
\cmidrule(lr){2-4}
& \textbf{1} & \textbf{2} & \textbf{4} & \\
\midrule
12 & $75\!\pm\!4$ & $81\!\pm\!5$ & $70\!\pm\!4$ & $75\!\pm\!6$ \\
24 & $78\!\pm\!4$ & $79\!\pm\!4$ & $81\!\pm\!3$ & \underline{$79\!\pm\!4$} \\
48 & $85\!\pm\!3$ & \textbf{$87\!\pm\!4$} & $76\!\pm\!5$ & \bm{$83\!\pm\!6$} \\
\bottomrule
\end{tabular}

\medskip
{\footnotesize\emph{\added{(c) Success rate ($\%$) of CP-SO2 with varying branch size.}}}
\end{table}

\begin{table}[htbp]
\caption{Comparison of success rates (\%) between Diffusion and Flow Matching on Stack D1 and Push T.}
\label{tab:ab_policy}
\centering
\renewcommand{\arraystretch}{1.2}
\setlength{\tabcolsep}{3pt}
\newcolumntype{L}[1]{>{\raggedright\arraybackslash}p{#1}} %
\newcolumntype{C}[1]{>{\centering\arraybackslash}p{#1}}    %

\begin{tabular}{L{1.3cm} C{1.2cm} C{1.8cm} C{1.2cm} C{1.8cm}}
\toprule
\multirow{2}{*}{} & \multicolumn{2}{c}{\textbf{Stack D1}} & \multicolumn{2}{c}{\textbf{Push T}} \\
\cmidrule(lr){2-3} \cmidrule(lr){4-5}
& Diffusion & Flow Matching & Diffusion & Flow Matching \\
\midrule
DP3 & $20\pm2$ & $17\pm4$ & $70\pm1$ & $83\pm2$ \\
iDP3 & $31\pm5$ & $20\pm5$ & $76\pm2$ & $84\pm1$ \\
Equibot & $1\pm1$ & $1\pm1$ & $74\pm3$ & $82\pm2$ \\
\rowcolor{gray!30}
CP-SO3 & \bm{$76\pm4$} & \bm{$63\pm9$} & \underline{$85\pm2$} & \underline{$85\pm2$} \\
\rowcolor{gray!30}
CP-SO2 & \bm{$76\pm6$} & \underline{$59\pm6$} & \bm{$87\pm3$} & \bm{$97\pm1$} \\
\bottomrule
\end{tabular}

\vspace{0.3em}
\footnotesize
\parbox{0.93\linewidth}{
Mean success rates are averaged over the final 10 checkpoints, each evaluated in 50 environment initializations with random seed 42.
}
\end{table}

\subsection{\added{Reliability Analysis of the Rotation Estimator}}
\label{sec:visual}
\added{In Section~\ref{sec:theory}, we introduced the theoretical foundation of the canonical policy, based on the concept of equivariant groups. Section~\ref{sec:preprocessing} further proves that all elements within an equivariant group can be transformed into a shared canonical representation. However, this definition is highly restrictive: even a tiny deviation, such as a single point difference, may break the rigid-transformation constraint and cause point clouds to fall into different groups, which introduces more variation for the policy. In practice, neural networks exhibit robustness to such noise through parameter optimization.} 

\added{We design an experiment to evaluate the reliability of the $\mathrm{SO}(3)$-equivariant network shown in Fig.~\ref{fig:SO3} under noisy inputs. Specifically, we sample eight point clouds from the Stack D1 task and apply eight augmentations per sample under four noise levels:}
\begin{itemize}
    \item \textbf{Level 0}: random $\mathrm{SO}(3)$ rotation only;
    \item \textbf{Level 1}: random rotation, pointwise Gaussian jitter ($\mu=0,\;\sigma=0.05$), and random dropout, cropping, and insertion, where each of these three operations is applied to 5\% of the total points;
    \item \textbf{Level 2}: same with $\sigma=0.1$ and 10\% perturbations;
    \item \textbf{Level 3}: same with $\sigma=0.2$ and 20\% perturbations.
\end{itemize}

\added{We compare the network outputs under two conditions: (i) \textit{frozen} parameters; and (ii) \textit{trained} parameters. UMAP~\cite{umap} is used to project the resulting equivariant features $\mathbf{F}_{\text{equiv}}$ (shown in Fig.~\ref{fig:SO3}) into 2D for visualization.}

\added{As shown in Fig.~\ref{fig:vis_feat}, with frozen parameters the features remain unchanged during the training process. At Level 0, features from all eight augmentations cluster tightly for both settings, reflecting the network’s inherent $\mathrm{SO}(3)$-equivariance. As noise increases, the frozen network’s features rapidly lose their cluster structure, indicating a lack of robustness. In contrast, the trained network maintains coherent clusters at Level 1 and Level 2, showing improved tolerance to moderate perturbations. However, at Level 3 (strong jitter and 20\% point violations), even the trained network fails to preserve feature consistency.}

\added{Quantitative results in Table~\ref{tab:forzon} confirm these trends: without training, the estimator strictly follows equivariance but cannot handle noisy inputs; with training, it becomes more robust and reliably maps approximately rigid point clouds to a shared canonical representation, thereby improving downstream policy performance.}

\begin{table}[t]
\centering
\caption{\added{Ablation of freezing versus not freezing the \(\mathrm{SO}(3)\)-equivariant network on canonical policy performance across tasks.}}
\label{tab:forzon}
\small
\setlength{\tabcolsep}{3.6pt}
\renewcommand{\arraystretch}{1.15}

\begin{tabular}{>{\centering\arraybackslash}p{1.0cm}  
                >{\centering\arraybackslash}p{0.9cm}  
                >{\centering\arraybackslash}p{1.3cm}
                >{\centering\arraybackslash}p{1.3cm}
                >{\centering\arraybackslash}p{1.3cm}
                >{\centering\arraybackslash}p{1.4cm}}
\toprule
&Frozen&
\textbf{Stack D1} & \textbf{Mug. D1} & \textbf{Ham.\ D1} & \textbf{Coffee D2} \\
\midrule
\multirow{2}{*}{CP-SO3}
 & $\times$ & \bm{$68\pm4$} & \bm{$37\pm5$} & \bm{$68\pm5$} & \bm{$40\pm7$} \\
 & \checkmark  & $18\pm5$ & $17\pm6$ & $31\pm7$ & $3\pm2$ \\
\midrule
\multirow{2}{*}{CP-SO2}
 & $\times$ & \bm{$76\pm5$} & \bm{$44\pm7$} & \bm{$67\pm8$} & \bm{$52\pm8$} \\
 & \checkmark  & $68\pm6$ & $37\pm6$ & $60\pm7$ & $37\pm6$ \\
\bottomrule
\end{tabular}
\vspace{0.3em}
\footnotesize
\parbox{0.93\linewidth}{\added{
Mean success rates are averaged over the final 10 checkpoints, each evaluated in 50 environment initializations with random seed 42.}}
\end{table}

\begin{figure*}[t]
    \centering
    \includegraphics[width=1.0\linewidth]{images/vis_feat.pdf}
    \caption{\added{Visualization of equivariant features produced by the $\mathrm{SO}(3)$-equivariant network under different noise levels and training stages.}}
    \label{fig:vis_feat}
\end{figure*}

\section{Real Robot Experiments}
In this section, we evaluate canonical policy on two real-world platforms: the Franka Emika Panda and UR5. We describe the experimental setup, compare against both image-based and point cloud-based baselines, and assess generalization under appearance and viewpoint shifts, as well as data efficiency in low-data regimes\added{, followed by an analysis of failure case}.

\subsection{Real-World Experimental Setup}
\begin{figure}[t]
    \centering
    \includegraphics[width=1.0\linewidth]{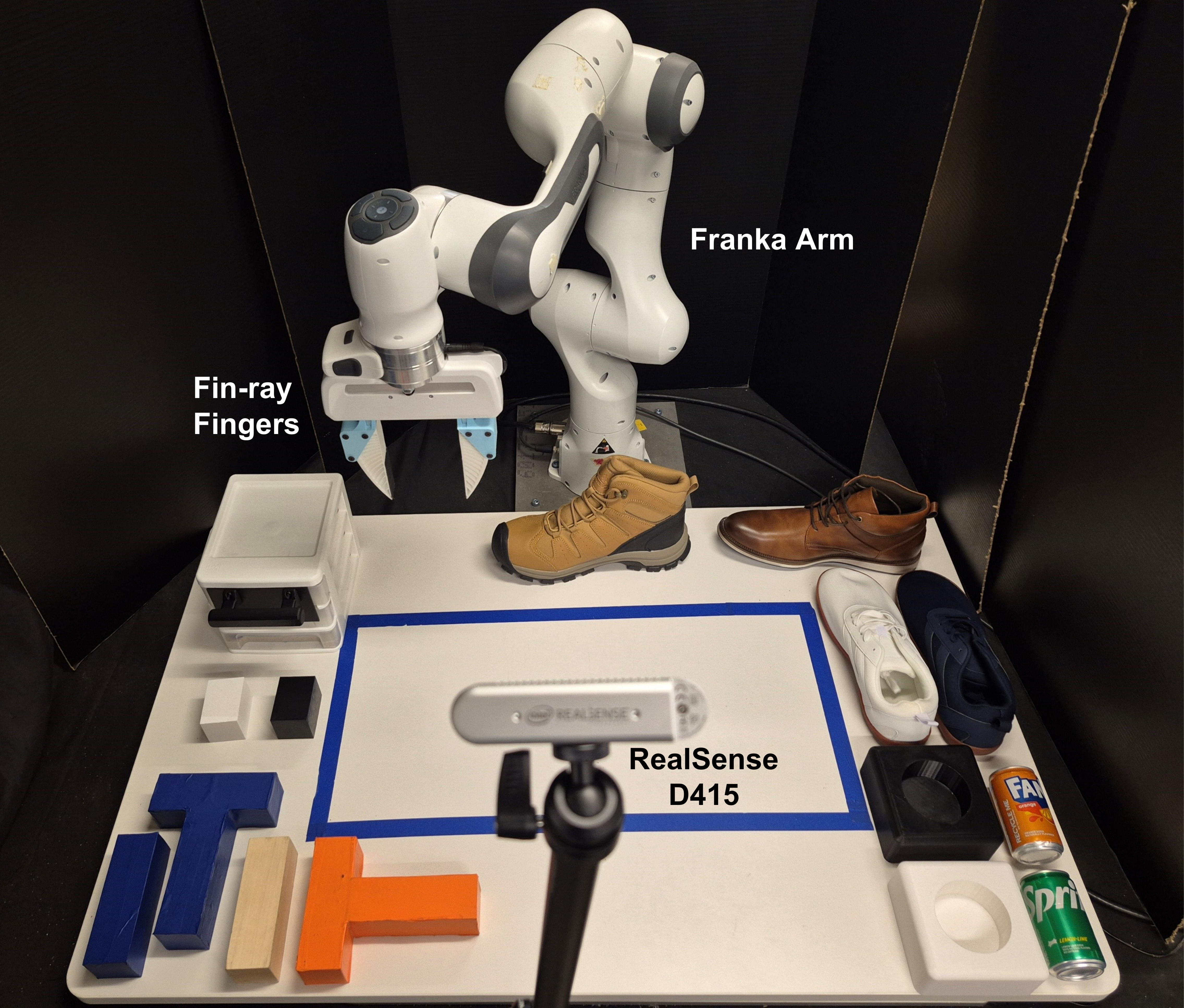}
    \caption{Real-world experimental setup on the Franka platform. A RealSense D415 camera provides visual observations. The objects shown are used across all evaluated tasks.}
    \label{fig:panda_setup}
\end{figure}
\begin{figure*}[t]
    \centering
    \includegraphics[width=1.0\linewidth]{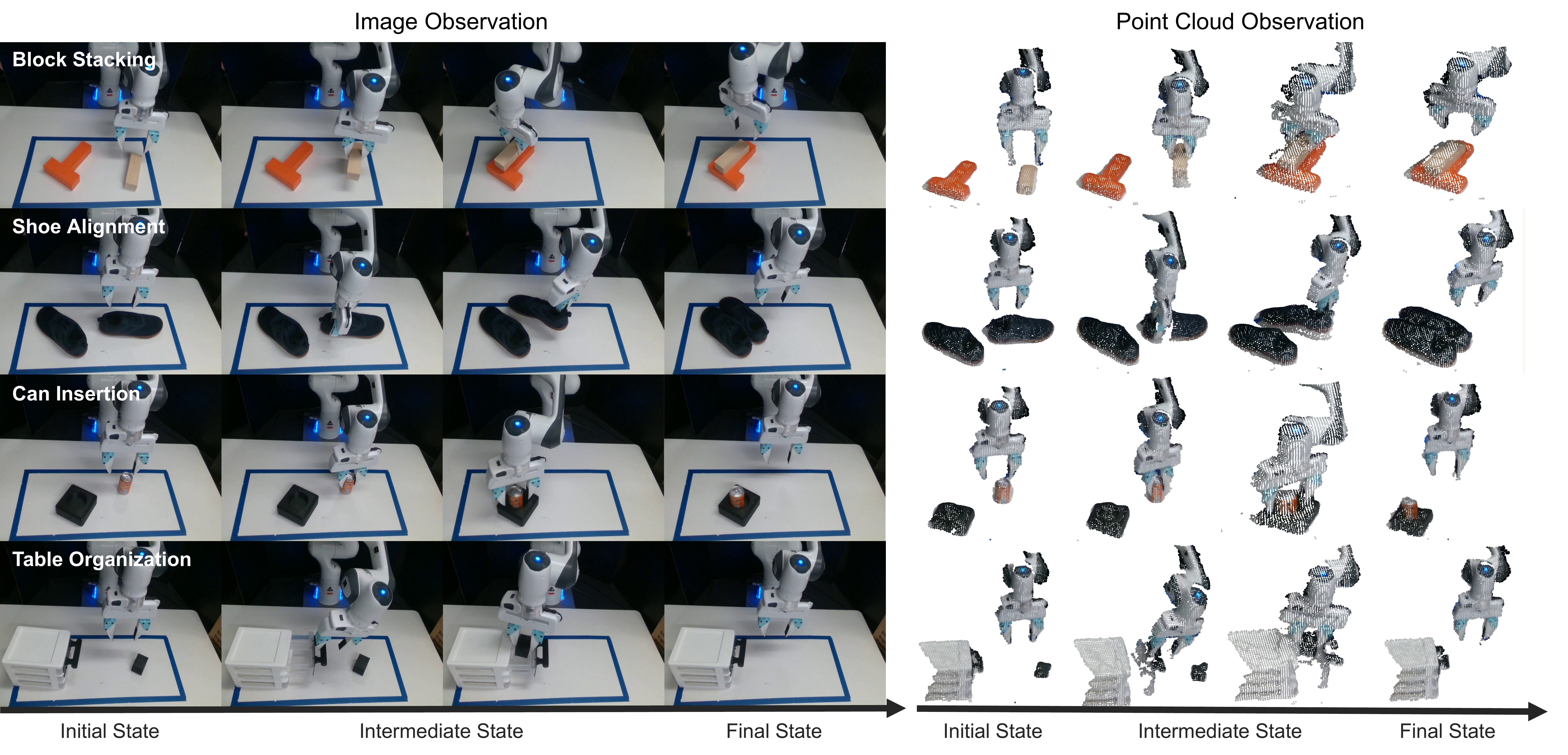}
    \caption{Real-world task executions on the Franka platform. The four tasks, namely Block Stacking, Shoe Alignment, Can Insertion, and Table Organization, are shown in three stages: initial, intermediate, and final. Both RGB image observations (left) and point cloud observations (right) are visualized to illustrate the input modalities used by the policy.}
    \label{fig:panda_tasks}
\end{figure*}

\begin{figure}[t]
    \centering
    \includegraphics[width=1.0\linewidth]{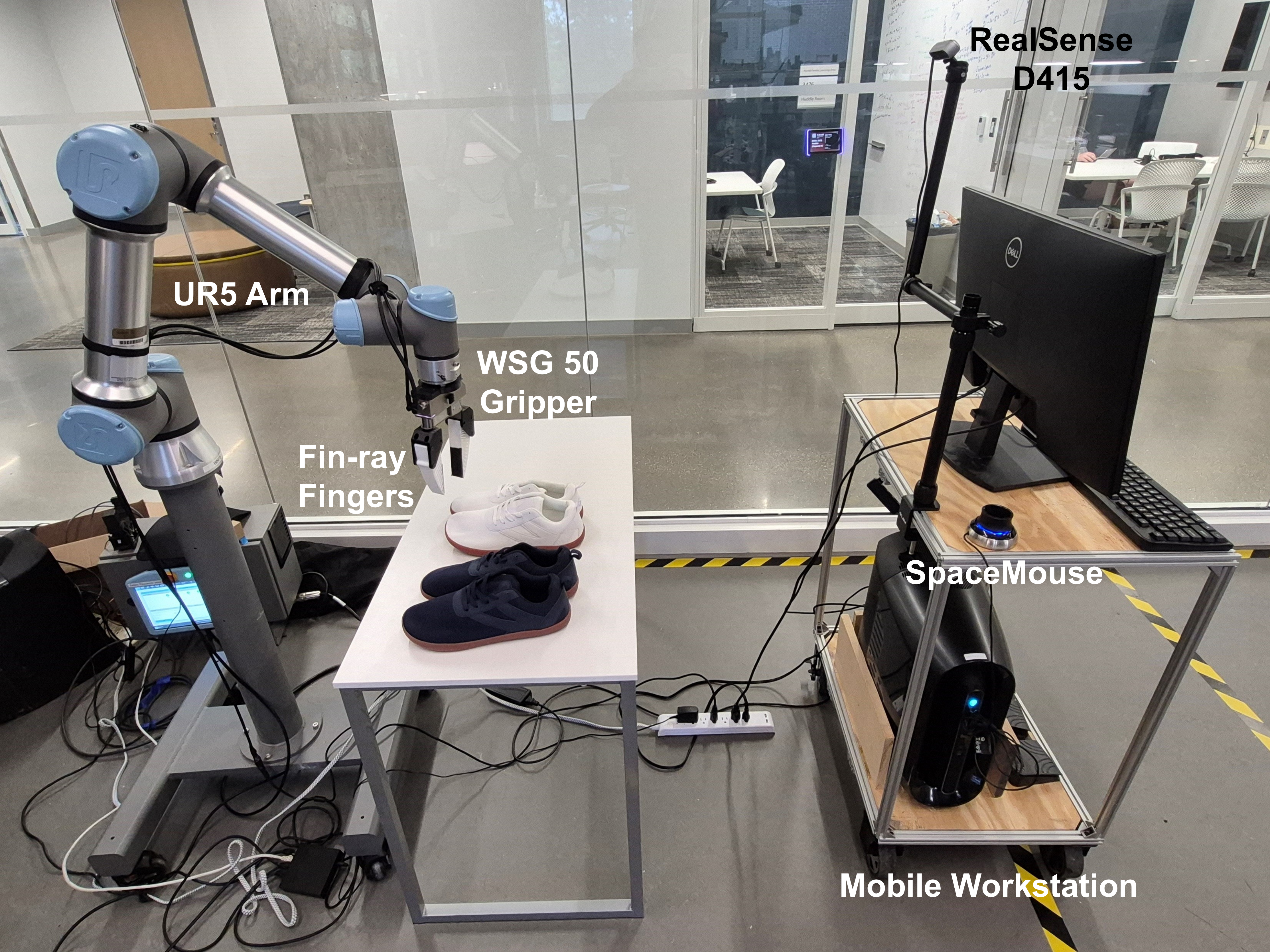}
    \caption{UR5 platform setup for real-world experiments. We employ a mobile workstation to enable flexible camera positioning and use a RealSense D415 for visual perception.}
    \label{fig:UR5_setup}
\end{figure}

We begin by describing the real-world experimental setup for both robotic platforms.

\textbf{Franka Emika Panda:} The first system consists of a Franka Emika Panda robot arm~\cite{panda} using a \textit{relative control} (see Section~\ref{sec:overview}), equipped with a pair of Fin-Ray fingers~\cite{finray} and a fixed Intel RealSense D415~\cite{realsense} RGB-D camera. Human demonstrations are collected via a 6-DoF 3DConnexion SpaceMouse~\cite{spacemouse}, with observations and actions recorded at 10~Hz. During policy rollout, we apply DDIM~\cite{DDIM} sampling with 20 denoising steps for faster inference.

The experimental setup and objects used are shown in Fig.~\ref{fig:panda_setup}. The D415 camera is mounted at a fixed location on the table, providing an angled view of the workspace. All interactions are constrained within the blue boundary illustrated in the figure.
We evaluate four manipulation tasks on this platform (see Fig.~\ref{fig:panda_tasks}): Block Stacking, Shoe Alignment, Can Insertion, and Table Organization. Tasks are arranged in increasing order of difficulty. Block Stacking requires coarse positional and orientational control, while Shoe Alignment involves precise orientation alignment between a pair of shoes. Can Insertion demands accurate localization and vertical insertion of a cylindrical object. Table Organization is the most challenging, involving a multi-step sequence with precise grasping and drawer manipulation.

We visualize the extracted point clouds on the right side of Fig.~\ref{fig:panda_tasks}. RGB-D frames are converted into 3D point clouds using calibrated camera intrinsics, followed by spatial cropping within a bounding box of size $0.95 \times 0.50 \times 0.74~\mathrm{m}^3$. To enhance generalization, only the $\mathit{xyz}$ coordinates are retained, discarding color information. Consistent with the simulation setup, all point clouds are uniformly downsampled to 256 points to reduce memory usage and improve inference efficiency. For each task, we collect 100 human demonstrations with randomized object placements.

To further assess generalization, we introduce distribution shifts by varying object appearance, including changes in shape and color, as illustrated in Fig.~\ref{fig:panda_appearance_change}.

We compare our method with the representative point cloud baseline DP3~\cite{DP3}, and use CP-SO2 as our canonical policy variant due to the upright alignment of the observation $z$-axis. To assess performance across modalities, we additionally evaluate two state-of-the-art image-based baselines: Diffusion Policy (DP)~\cite{diffusion_policy} and the equivariant variant EquiDiff~\cite{equidiff}. Input images are resized from $640 \times 480$ to $256 \times 256$, then center-cropped to $230 \times 230$ before being fed into the encoder.

\textbf{UR5:} To test under different control and sensing conditions, we also conduct experiments on a UR5~\cite{ur5} robot using \textit{absolute control} (see Section~\ref{sec:overview}). As shown in Fig.~\ref{fig:UR5_setup}, the system is equipped with a WSG 50 gripper~\cite{wsg50} for precise width control and Fin-Ray fingers for compliant grasping. A single RealSense D415 is mounted on a mobile workstation to provide RGB-D input. The mobility of the camera facilitates easy variation of viewpoints. Demonstrations are collected via the SpaceMouse. We follow the same point cloud generation and cropping procedure as described earlier.
\added{In addition, we evaluate ACT~\cite{zhao2023learning} as an image-based baseline on the UR5, given that its chunked absolute end-effector outputs are compatible with the UR5 control mode.}

On the UR5 platform, we conduct experiments to evaluate the data efficiency and robustness to camera viewpoint changes of canonical policy, comparing it against both image and point cloud policies\added{, and additionally analyzing failure cases}.

\subsection{Benchmarking Across Diverse Tasks}

\begin{figure}[t]
    \centering
    \includegraphics[width=1.0\linewidth]{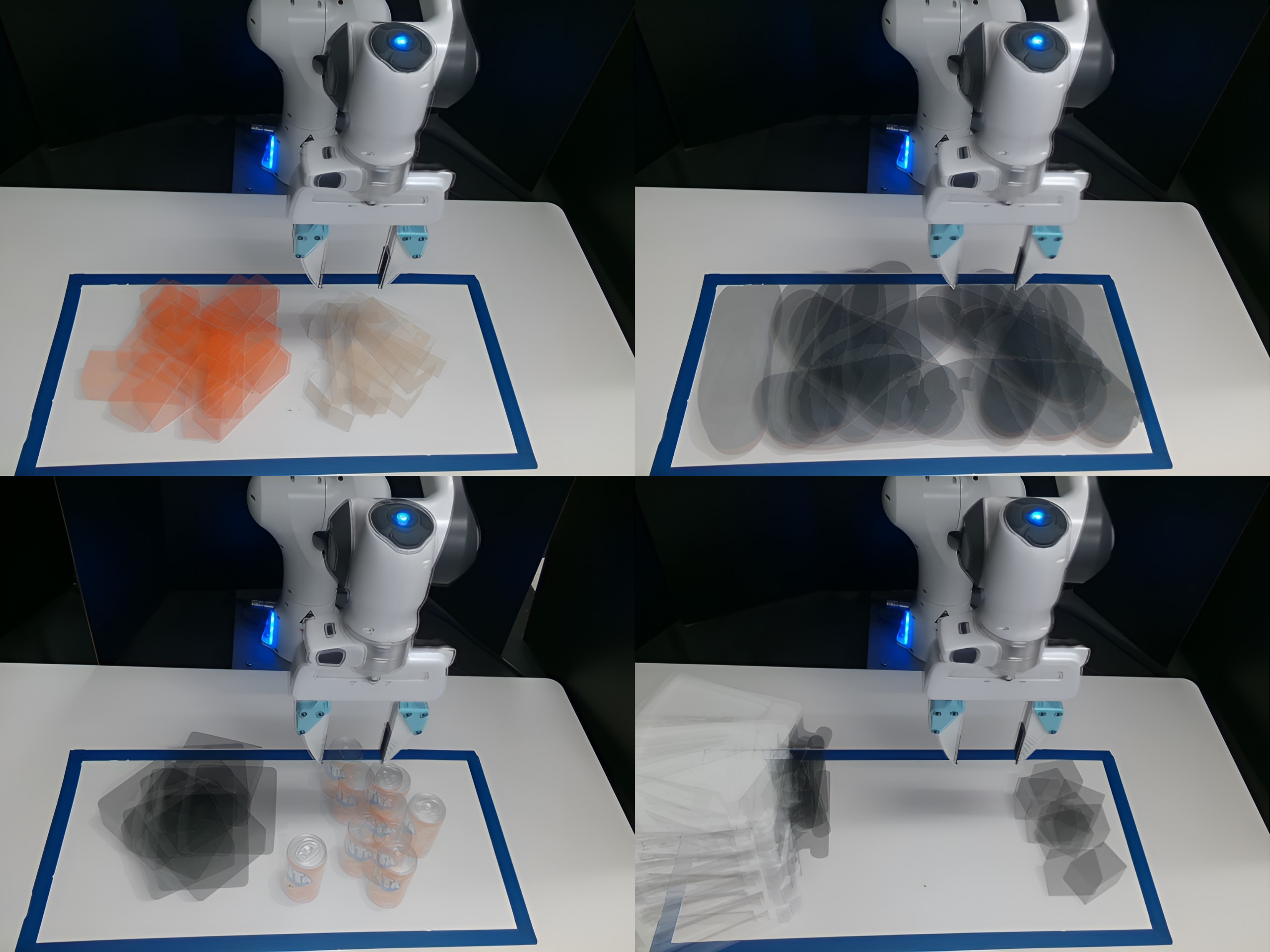}
    \caption{Randomized initial object configurations used to evaluate policy generalization and robustness across tasks.}
    \label{fig:random}
\end{figure}

\begin{table}[t]
\centering
\caption{Success rates over 10 trials for real-world tasks on the Franka Emika platform.}
\label{tab:realworld_normal}
\setlength{\tabcolsep}{4pt}          
\renewcommand{\arraystretch}{1.125}  
\begin{tabular}{l l|c c c c}
\toprule
\textbf{Modality} & \textbf{Policy} &
\makecell{Block \\ Stacking} &
\makecell{Shoe \\ Alignment} &
\makecell{Can \\ Insertion} &
\makecell{Table \\ Organization} \\
\midrule
\multirow{2}{*}{Image}
  & DP        & 2/10 & 1/10 & \textbf{4/10} & 0/10 \\
  & Equidiff  & \underline{3/10} & 3/10 & 3/10 & \underline{2/10} \\
\midrule
\multirow{2}{*}{Point Cloud}
  & DP3       & \underline{3/10} & \underline{5/10} & 1/10 & 0/10 \\
  & \cellcolor{gray!30}CP\mbox{-}SO2
              & \cellcolor{gray!30}\textbf{7/10}
              & \cellcolor{gray!30}\textbf{6/10}
              & \cellcolor{gray!30}\textbf{4/10}
              & \cellcolor{gray!30}\textbf{3/10} \\
\bottomrule
\end{tabular}
\end{table}

Table~\ref{tab:realworld_normal} summarizes the quantitative performance of different policies with various input modalities across four real-world tasks, evaluated over 10 trials on the Franka platform. CP-SO2 consistently achieves the highest success rates compared to other policies as task difficulty increases. For simpler tasks such as Block Stacking and Shoe Alignment, CP-SO2 achieves 70\% and 60\% success rates, outperforming the second-best policies by 40\% and 10\%, respectively. Besides, in these two tasks, policies based on point cloud modalities achieve higher success rates compared to those relying on image modalities.

For the Can Insertion task, DP achieves comparable results to CP-SO2, with both attaining a 40\% success rate, which is 10\% higher than EquiDiff and 30\% higher than DP3. However, for the long-horizon Table Organization task, only EquiDiff and CP-SO2 achieve non-zero success rates, with CP-SO2 reaching 30\%, outperforming EquiDiff by 10\%.

Fig.~\ref{fig:random} illustrates the variability of policy rollouts over 10 trials. We sum the initial state observation of each trial for CP-SO2 and compute the average, which confirms that each trial differs from the others, ensuring sufficient randomness for fair evaluation. Therefore, the quantitative results reported in Table~\ref{tab:realworld_normal} validate the effectiveness of the proposed canonical representation, which outperforms baselines across both point cloud and image modalities in real-world scenarios.

Finally, it is worth noting that the results on the Block Stacking and Table Organization tasks are consistent with their corresponding tasks in MimicGen (Stack D1 and Mug Cleanup D1), indicating strong sim-to-real consistency in sensory observations.

\subsection{Generalization to Unseen Object Appearances}
\begin{figure*}[t]
    \centering
    \includegraphics[width=1.0\linewidth]{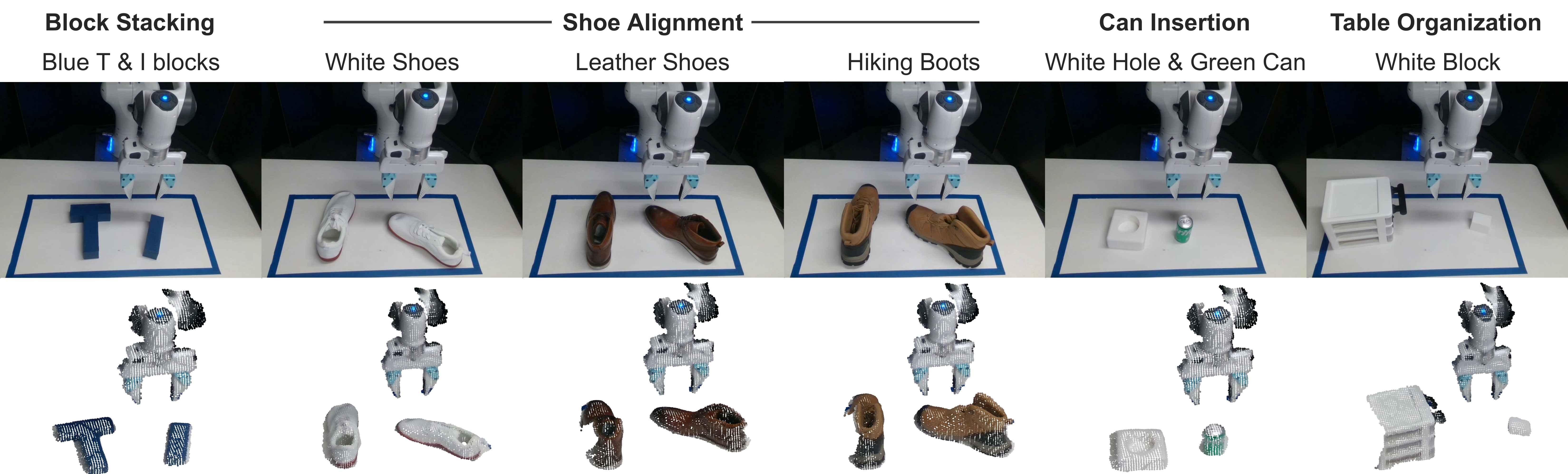}
    \caption{Visualization of appearance variation experiments on the Franka platform.}
    \label{fig:panda_appearance_change}
\end{figure*}

\begin{figure*}[t]
    \centering
    \includegraphics[width=1.0\linewidth]{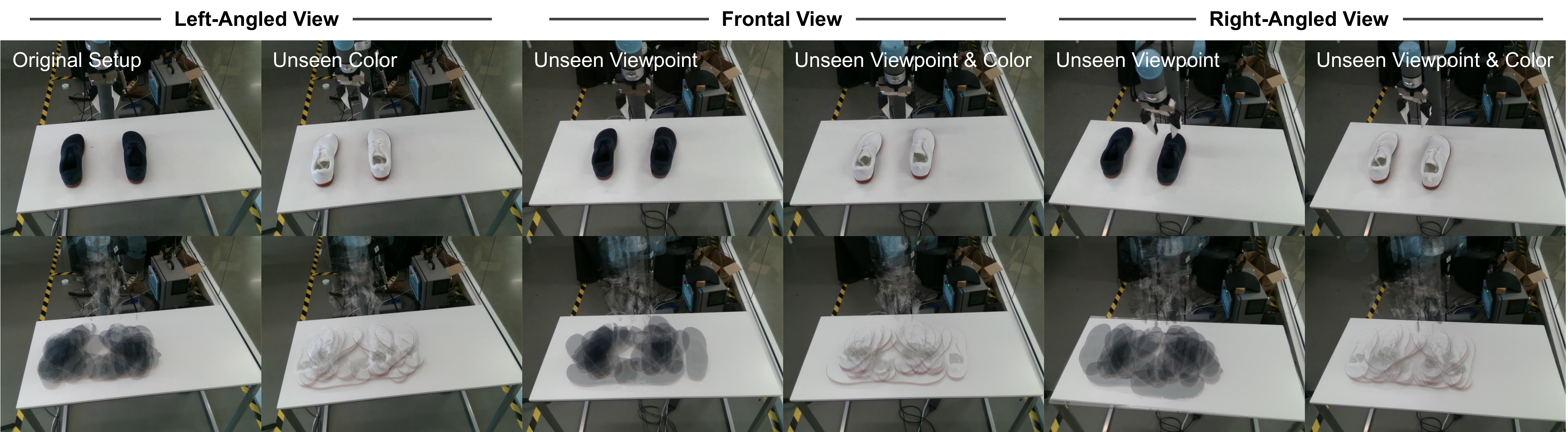}
    \caption{Visualization of the Shoe Alignment task on the UR5 platform under appearance and viewpoint variations.}
    \label{fig:UR5_tasks}
\end{figure*}

\begin{table}[t]
\centering
\caption{Success rates over 10 trials under color variations on the Franka Emika platform.}
\label{tab:color_change}
\setlength{\tabcolsep}{4pt}
\renewcommand{\arraystretch}{1.125} 
\begin{tabular}{l l|c c c c} 
\toprule
\textbf{Modality} & \textbf{Policy} & \makecell{Block \\ Stacking} & \makecell{Shoe \\ Alignment} & \makecell{Can \\ Insertion} & \makecell{Table \\ Organization} \\
\midrule
\multirow{2}{*}{Image} 
    & DP       & 0/10 & 0/10 & 0/10 & 0/10 \\
    & Equidiff & 0/10 & 0/10 & \underline{3/10} & 0/10 \\
\midrule
\multirow{2}{*}{Point Cloud} 
    & DP3      & \underline{3/10} & \underline{3/10} & 1/10 & 0/10 \\
    & \cellcolor{gray!30}CP-SO2 & \cellcolor{gray!30}\textbf{6/10} & \cellcolor{gray!30}\textbf{6/10} & \cellcolor{gray!30}\textbf{4/10} & \cellcolor{gray!30}\textbf{2/10} \\
\bottomrule
\end{tabular}
\end{table}

\begin{table}[t]
\centering
\caption{Success rates over 10 trials under both color and shape variations on the Franka Emika platform.}
\label{tab:shape_change}
\begin{tabular}{ll|cc}
\toprule
Modality & Policy & \textbf{Leather Shoes} & \textbf{Hiking Boots} \\
\midrule
\multirow{2}{*}{Image} 
    & DP        & 0/10 & 0/10 \\
    & Equidiff  & 2/10 & 0/10 \\
\midrule
\multirow{2}{*}{Point Cloud} 
    & DP3       & \underline{3/10} & \textbf{2/10} \\
    & \cellcolor{gray!30}CP-SO2    & \cellcolor{gray!30}\textbf{5/10} & \cellcolor{gray!30}\textbf{2/10} \\
\bottomrule
\end{tabular}
\end{table}

\begin{table}[t]
\centering
\caption{Success rates over 10 trials under varying numbers of demonstrations and shoe colors on the UR5 platform.}
\label{tab:data_efficiency}
\scriptsize   
\setlength{\tabcolsep}{3.6pt} 
\renewcommand{\arraystretch}{1.125} 
\begin{tabular}{ll|cc|cc}
\toprule
\multirow{2}{*}{Modality} & \multirow{2}{*}{Policy} & 
\multicolumn{2}{c}{\textbf{50 Demos}} & 
\multicolumn{2}{c}{\textbf{100 Demos}} \\
\cmidrule(lr{0.4em}){3-4} \cmidrule(lr{0.4em}){5-6}
& & Black Shoes & White Shoes & Black Shoes & White Shoes \\
\midrule
\multirow{3}{*}{Image}
    & \added{ACT}       & \added{0/10} & \added{0/10} & \added{1/10} & \added{1/10} \\
    & DP        & 0/10 & 0/10 & 0/10 & 0/10 \\
    & Equidiff  & 0/10 & 0/10 & 0/10 & 0/10 \\
\midrule
\multirow{2}{*}{Point Cloud} 
    & DP3       & 0/10 & 0/10 & \underline{3/10} & \underline{2/10} \\
    & \cellcolor{gray!20}CP-SO2 & \cellcolor{gray!20}\textbf{4/10} & \cellcolor{gray!20}\textbf{6/10} & \cellcolor{gray!20}\textbf{6/10} & \cellcolor{gray!20}\textbf{7/10} \\
\bottomrule
\end{tabular}
\end{table}

\begin{table*}[t]
\centering
\caption{Success rates over 10 trials under varying camera viewpoints and shoe colors on the UR5 platform.}
\label{tab:robustness_eval}
\begin{tabular}{ll|cccccc}
\toprule
\multirow{2}{*}{Modality} & \multirow{2}{*}{Policy} & 
\multicolumn{2}{c}{\textbf{Left-Angled View}} & 
\multicolumn{2}{c}{\textbf{Frontal View}} & 
\multicolumn{2}{c}{\textbf{Right-Angled View}} \\
\cmidrule(lr{0.4em}){3-4} \cmidrule(lr{0.4em}){5-6} \cmidrule(lr{0.4em}){7-8}
& & Black Shoes & White Shoes & Black Shoes & White Shoes & Black Shoes & White Shoes \\
\midrule
\multirow{3}{*}{Image}
    & \added{ACT}       & \added{0/10} & \added{0/10} & \added{0/10} & \added{0/10} & \added{0/10} & \added{0/10} \\
    & DP       & 0/10 & 0/10 & 0/10 & 0/10 & 0/10 & 0/10 \\
    & Equidiff & 0/10 & 0/10 & 0/10 & 0/10 & 0/10 & 0/10 \\
\midrule
\multirow{2}{*}{Point Cloud} 
    & DP3     & 0/10 & 0/10 & 0/10 & 0/10 & 0/10 & 0/10 \\
    & \cellcolor{gray!20}CP-SO2 & 
    \cellcolor{gray!20}\textbf{4/10} & \cellcolor{gray!20}\textbf{6/10} & 
    \cellcolor{gray!20}\textbf{4/10} & \cellcolor{gray!20}\textbf{6/10} & 
    \cellcolor{gray!20}\textbf{3/10} & \cellcolor{gray!20}\textbf{3/10} \\
\bottomrule
\end{tabular}
\end{table*}

\added{In this section, to thoroughly evaluate the generalizability of canonical policy operating on point cloud inputs, we modify each task by altering object colors and shapes.}
The same trained policies for each task are then used to perform rollouts in the modified environments. The appearance changes are illustrated in Fig.~\ref{fig:panda_appearance_change}.

Specifically, in the Block Stacking task, both the I-shaped and T-shaped blocks are recolored to blue, representing a significant color shift from their original appearance. In the Shoe Alignment task, the shoes used during training are replaced with white shoes. Furthermore, two additional shoe types: leather shoes and hiking boots are introduced, differing substantially in both color and shape from the original shoes. In the Can Insertion task, the original black holder and orange can are replaced with a white holder and green can, respectively. Finally, in the Table Organization task, only the color of the small block is changed from black to white.
\added{Each new scenario is evaluated over 10 trials, with randomness controlled in the same manner as in Fig.~\ref{fig:random}.}

The experimental results under color variations are summarized in Table~\ref{tab:color_change}. Since point cloud-based policies are trained without color information, their performance under appearance changes is significantly more robust compared to image-based policies, as evidenced by the comparison with Table~\ref{tab:realworld_normal}. Notably, CP-SO2 consistently achieves the highest success rates across all tasks, outperforming the second-best policy by 30\%, 30\%, 10\%, and 20\% for tasks of increasing difficulty. Interestingly, Equidiff achieves the second-best performance in the Can Insertion task, outperforming DP3 by 20\%, but shows zero success rates in other tasks, similar to DP. This phenomenon indicates that while Equidiff exhibits some degree of generalization, it remains inferior to canonical policy, as CP-SO2 maintains stable and high performance across all tasks under appearance changes.

Table~\ref{tab:shape_change} further assesses the generalization capability of each policy under simultaneous changes in both color and shape. CP-SO2 consistently achieves the highest success rates, outperforming the second-best policy by 20\% in the Leather Shoe Alignment task. In the Hiking Boot Alignment task, DP3 achieves performance comparable to CP-SO2.
This observation can be explained by the fact that, in the original training set, the shoes were low-top, leading the end-effector to learn grasping motions at relatively low heights. When encountering high-top hiking boots, the end-effector often follows a lower approach trajectory and collides with the boot during the grasping attempt, significantly increasing the difficulty of successful grasping.
Under such challenging conditions, the performance of point cloud-based policies remains similar.

\subsection{\added{Data Efficiency Analysis}}
In addition to evaluating policy performance under nominal conditions, we assess the generalization capability of canonical policy in real-world settings with limited data and varying observation conditions. Specifically, we study two challenging scenarios on the UR5 platform: (i) data efficiency under varying numbers of demonstrations, and (ii) robustness to camera viewpoint changes \added{(discussed in the next subsection)}.

These settings simulate realistic deployment challenges where data collection may be constrained, and sensing configurations may vary over time. We show that CP-SO2, which leverages SO(2)-equivariance and point cloud inputs, consistently generalizes better than both image-based and alternative point cloud baselines.

We select the Shoe Alignment task, using two visually distinct objects: a pair of black shoes and a pair of white shoes. All human demonstrations and training data are collected using the black shoes, while the trained policy is evaluated through rollouts on both black and white shoes.

Table~\ref{tab:data_efficiency} summarizes the results under varying numbers of training demonstrations, comparing each method using 50 and 100 human demonstrations collected with black shoes. Each configuration is evaluated over 10 rollout trials. When trained with only 50 demonstrations, CP-SO2 is the only method that achieves non-zero success rates, reaching 40\% on black shoes and 60\% on white shoes. In contrast, all other baselines, including the image-based policies \added{ACT}, DP and Equidiff as well as the point cloud-based baseline DP3, fail completely under this low-data setting.

As the number of demonstrations increases to 100, CP-SO2's performance further improves to 60\% and 70\% on black and white shoes, respectively, outperforming the second-best method DP3 by 30\% and 50\%.
Notably, while CP-SO2 and DP3 exhibit clear improvements with additional training data, the performance of image-based policies remains \added{essentially unchanged}. 

As shown in Tables~\ref{tab:realworld_normal} and \ref{tab:data_efficiency}, all baselines (DP, DP3, and EquiDiff) perform worse on the UR5 compared to the Franka Panda. We attribute this to differences in the low-level control setups. In the Franka Panda experiments, we utilized an Operational Space Controller, which helped smooth out the jerky actions produced by the learned policies. In contrast, the UR5 setup relied on high-precision pose tracking without such smoothing. As illustrated in our supplementary video, under a single-camera setup and with only 50 or 100 demonstrations, baseline policies frequently produced unstable motions on the UR5, leading to degraded task performance. These results further \added{highlight} the consistent and robust performance of canonical policy across different robotic platforms.

An interesting observation is that CP-SO2 achieves higher success rates on white shoes, despite being trained exclusively on black shoes, underscoring the strong generalization capability of canonical policy. From a point cloud perspective, black and white shoes share the same geometric structure, enabling a policy trained on black shoes to naturally generalize to white ones. Nevertheless, the point clouds of white shoes introduce subtle variations compared to black shoes in the training set. We believe that this slight randomness in observations may actually enhance policy robustness, a phenomenon frequently observed in imitation learning, potentially explaining the improved performance on white shoes. Introducing controlled randomness into observations is a common strategy to improve robustness of policies, as seen in techniques such as random image cropping in Diffusion Policy \cite{diffusion_policy} and color jitter or random rotations in UMI training \cite{chi2024universal}.

\subsection{\added{Robustness to Viewpoint Shifts and Failure Case Analysis}}
\added{
To evaluate the robustness of canonical policy under camera viewpoint changes, we fix the number of training demonstrations at 50 and rotate the mobile workstation to simulate different viewpoints. As shown in Fig.~\ref{fig:UR5_tasks}, the leftmost two columns (Left-Angled View) correspond to the original training setup. The middle columns (Frontal View) simulate a moderate $15^\circ$ rotation, while the rightmost columns (Right-Angled View) represent a larger rotation of approximately $30^\circ$, which induces a stronger deviation from the training viewpoint.
}

\added{A change in camera viewpoint is theoretically equivalent to an inverse transformation of the scene while keeping the camera fixed. This experimental setup thus enables us to validate the theoretical foundation of our canonical representation. As the mobile workstation performs rigid-body translations and in-plane rotations, the $z$-axis remains consistently aligned upright, ensuring the applicability of CP-SO2, which is designed to be $\mathrm{SO}(2)$-equivariant about the vertical axis.}

\added{Table~\ref{tab:robustness_eval} summarizes the success rates under different viewpoint shifts. When the viewpoint is shifted to the Frontal View, CP-SO2 maintains strong performance with success rates of 40\% and 60\% for black and white shoes, respectively. In contrast, all baseline policies fail completely under the same conditions. When the camera is further rotated to the Right-Angled View, the performance of CP-SO2 drops to 30\% for both shoe types.}

\added{
This trend mirrors the findings in Section~\ref{sec:visual} on the noise robustness of the $\mathrm{SO}(3)$-equivariant network. Moderate viewpoint shifts (Frontal View) resemble the low-to-moderate noise levels (Levels 1–2 in Fig.~\ref{fig:vis_feat}), where parameter optimization enables the network to effectively map new observations back to the canonical frame of the original setup. Consequently, the canonicalization remains consistent and the downstream policy achieves comparable performance.}

\added{
However, larger viewpoint shifts (Right-Angled View) introduce more severe disruptions: significant occlusion, previously unseen regions, and larger changes in visible geometry, which are similar to the high-noise condition (Level 3 in Fig.~\ref{fig:vis_feat}). Under such conditions, the network can no longer reliably align these altered point clouds to a single canonical space. As a result, observations that should be mapped to approximately the same canonical representation are instead projected into inconsistent frames. This increased input diversity leads to degraded policy performance.
}

\section{Conclusions and Limitations}
In this paper, we present canonical policy, a principled framework for 3D equivariant imitation learning that unifies point cloud observations through a canonical representation. Built upon a rigorous theory of 3D canonical mappings, our method enables end-to-end learning of spatially equivariant policies from demonstrations. By leveraging geometric consistency through canonicalization and the expressiveness of generative policy models, canonical policy improves generalization and data efficiency in imitation learning.

We validate our approach on 12 diverse simulated tasks and 4 real-world manipulation tasks, covering 16 evaluation settings involving distribution shifts in object appearance, shape, camera viewpoint, and robot platform. Canonical policy consistently outperforms existing image and point cloud baselines, demonstrating superior generalization capability and data efficiency.

Despite its strong empirical performance, canonical policy has several limitations. First, both the Point Cloud Aggregation Encoder Block and the SO(3)-equivariant networks used in our framework rely on neighborhood processing, which can be computationally expensive and time-consuming. This overhead may limit scalability in resource-constrained applications.

In addition, while SO(3)-equivariant networks are theoretically robust to global rotations, they remain sensitive to the specific structure of the input point cloud. Due to occlusions and viewpoint-dependent sampling, different views of the same object may result in substantially different point cloud geometries. As a result, the equivariant features produced from these inputs may differ significantly, which undermines the ability of canonical policy to generalize across large viewpoint shifts. Addressing this limitation may require further advances in viewpoint-invariant or view-consistent point cloud encoders, or the integration of multi-view representations.

\bibliographystyle{IEEEtran}
\bibliography{references}

\vfill

\end{document}